\pdfoutput=1
\documentclass{amsart}
\usepackage[utf8]{inputenc}
\usepackage[T1]{fontenc}
\usepackage{lmodern}
\usepackage{comment}

 \usepackage{amsaddr}
\usepackage{amscd,amsmath,amssymb,amsthm,amsfonts,epsfig,graphics}
\usepackage{graphicx}
\usepackage{subcaption}
\usepackage{epstopdf}
\usepackage{caption,subcaption}
\usepackage{float}
\usepackage[section]{placeins}
\usepackage{hyperref}
\usepackage[margin=1.3in]{geometry}
\hypersetup{
	colorlinks=true,
	linkcolor=black,
	filecolor=magenta,      
	urlcolor=black,
}

\newtheorem{theorem}{Theorem}
\newtheorem{lemma}{Lemma}

\newtheorem{proposition}{Proposition}

\theoremstyle{definition}
\newtheorem*{definition}{Definition}

\theoremstyle{remark}
\newtheorem{remark}{Remark}

\usepackage{float}
 
\newcommand{\cl}{\mathbb{R}}
\newcommand{\ddim}{\mathbb{R}^d}

\newcommand{\intnive}{\int_{0}^{\infty}}
\newcommand{\gauss}{exp \left( -t \norm{ \frac{x-x_i}{h}}  ^2   \right )}
\newcommand{\norm}[1]{\left\lVert#1\right\rVert}
\newcommand{\nd}{\noindent}
\newcommand{\xst}{x^{*}}

\newcommand{\kdxxih}{k'\left( \norm{ \frac{x-x_i}{h}}  ^2 \right)}
\newcommand{\kdxxihstar}{k'\left( \norm{ \frac{x^*-x_i}{h}}  ^2 \right)}
\newcommand{\kddxxih}{k''\left( \norm{ \frac{x-x_i}{h}}  ^2 \right)}
\newcommand{\kddxxihstar}{k''\left( \norm{ \frac{x^*-x_i}{h}}  ^2 \right)}
\newcommand{\covxxi}{(x-x_i)(x-x_i)^{T}}
\newcommand{\covxstxsti}{(x^*-x_i)(x^*-x_i)^{T}}
\newcommand{\sumn}{\sum_{i=1}^{n}}

\title[Convergence and clustering analysis for Mean Shift]{Convergence and clustering analysis for Mean Shift with radially symmetric, positive definite kernels}

\date{\today}
\begin{document}
\maketitle

\begin{center}
  \textbf{Susovan Pal} \\
  Department of Mathematics and Data Science, Vrije Universiteit Brussel (VUB) \\
  Pleinlaan 2, B-1050 Elsene/Ixelles, Belgium \\
  \texttt{susovan.pal@vub.be, susovan97@gmail.com}

\end{center}

\vspace{1em}

\begin{abstract}
The mean shift (MS) is a non-parametric, density-based, iterative algorithm with prominent usage in clustering and image segmentation. A rigorous proof for the convergence of its mode estimate sequence in full generality remains unknown. In this paper, we show that for\textit{ sufficiently large bandwidth} convergence is guaranteed in any dimension with \textit{any radially symmetric and strictly positive definite kernels}. Although the author acknowledges that our result is partially more restrictive than that of \cite{YT} due to the lower limit of the bandwidth, our kernel class is not covered by the kernel class in \cite{YT}, and the proof technique is different. Moreover, we show theoretically and experimentally that while for Gaussian kernel, accurate clustering at \textit{large bandwidths} is generally impossible, it may still be possible for other radially symmetric, strictly positive definite kernels.
\end{abstract}

\tableofcontents

\section{Introduction}\label{scn:intro}

\nd The mean shift (MS)\cite{FH} is a nonparametric, density-based, iterative algorithm used to find the modes of an estimated probability density function (pdf). It is closely related to the Expectation-Maximization (EM) algorithm for Gaussian Mixture Models (GMM), and is used in applications such as clustering, image segmentation (an application of clustering) \cite{CP}, and object tracking \cite{CP}. \\


\nd At a high level, the algorithm iteratively flows each data point along a trajectory to the zone of highest local maximum density of the unknown pdf and after some iterations, the iterated points start to form the clusters. To implement this, the algorithm first selects a data point $x$, calculates the weighted mean $w$ of the data points around $x$ in a ball of predefined radius (that is linked to the bandwidth used), then replaces $x$ by $w,$  and then iteratively recalculates the weighted sample mean around $w.$ The idea is that the shift from $x$ to $w,$ called 'mean shift', is supposed to be to the direction of the local maximum density, and  thus after a theoretically infinite number of iterations, the algorithm is supposed to converge to the modes of the unknown probability density function(pdf).\\

\nd Unlike the popular $K$-means clustering, we don't need to provide any predefined number of clusters ($K$), hence it is non-parametric. The bandwidth can be estimated from the data itself, and the algorithm automatically determines the number of clusters from the data. MS is credited to the early work of \cite{FH}, and a thorough overview can be found in \cite{CP}, where they gave a fairly comprehensive introduction to MS and its uses in computer vision and machine learning.\\

\nd Despite the use of the MS algorithm in different applications, a rigorous proof of convergence in the \textit{fully general case} is still missing in the literature. \cite{Gh1} did a fairly comprehensive survey of the incompleteness of the previous proofs in Section 3 and the reader is referred to it in case the reader wants to familiarize themselves with the history of the proofs of convergence. \cite{Gh2} proved the convergence in \textit{one dimension} with any kernel with a convex, differentiable, and strictly decreasing profile (the profile is to be defined in the next section). \cite{Gh1}, written by the same author, gave a sufficient condition for convergence with Gaussian kernels in any dimensions (note that Gaussian kernel has profile that is convex, differentiable and strictly decreasing). This opens up the possibility for proving or finding a sufficient condition to guarantee the convergence with more general kernels, in any dimensions. This paper takes a step towards this generalization, and in Theorem \ref{mainTheo}, we prove a sufficient condition guaranteeing convergence of the mean-shift mode-estimate sequence for any radially symmetric, strictly positive definite kernel when the bandwidth exceeds an explicit threshold depending on the kernel derivatives and data radius.

\subsection{Relevance of radially symmetric, positive definite kernels}
This subsection serves as a motivation to use positive (semi) definite kernels, although the main result assumes the kernels to be strictly positive definite. Radially symmetric positive (semi) definite kernels—kernels of the form
\( K(x,y) = k(\|x-y\|) \) or \( K(x,y) = k(\|x-y\|^{2}) \)
whose associated Gram matrices are positive (semi) definite for every finite
dataset—play a central role across machine learning and clustering, and their
use is not ad hoc.  Radial symmetry 
induces an isotropic, distance-based notion of similarity that is invariant
under Euclidean isometries, a natural modeling assumption in many unsupervised
learning and clustering problems.
Positive (semi) definiteness is precisely the condition that
guarantees an associated reproducing-kernel Hilbert space representation,
ensuring that the kernel admits an interpretation as an inner product in a
feature space thanks to the Moore–Aronszajn theorem \cite{Aronszajn1950}. See
\cite{ScholkopfSmola2002,ShaweTaylorCristianini2004} for more on these kernels. \\

In clustering, positive definite kernels
provide the standard framework in which kernel $k$-means can be interpreted as
Euclidean $k$-means in feature space, yielding rigorous connections to spectral
clustering and normalized-cut formulations
\cite{DhillonGuanKulis2004,vonLuxburg2007}. Finally, the class of radially symmetric,
positive definite kernels is broad yet structurally constrained: classical
results of Schoenberg (cf. Proposition \ref{prop:Schonbergh}) show that kernels of the form
\( K(x,y) = k(\|x-y\|^{2}) \) are positive (semi) definite on \( \mathbb{R}^d \)
for all dimensions if and only if the profile \( k\) is completely
monotone, thereby encompassing—beyond the Gaussian—widely used families such as
Laplace-type and Matérn kernels \cite{Schoenberg1938,Wendland2005}.

\subsection{Summary of the main result.} 
The main theorem \ref{mainTheo} of this paper establishes that for \textit{any} radially symmetric, strictly positive definite kernel in \textit{any} dimension---equivalently, for kernels whose profiles are completely monotone---the sequence of mean-shift mode estimates converges whenever the bandwidth parameter $h$ exceeds a \textit{computable} threshold $h_0$. This extends the classical Gaussian case proved in \cite{Gh1} to the entire class of strictly positive definite radially symmetric kernels.\\

\subsection{Comparison with recent work: scope and limitations of the present analysis in terms of bandwidths and kernel classes} One restriction of our main theoretical result, Theorem~\ref{mainTheo}, is that, like \cite{Gh1}, convergence is guaranteed for sufficiently large bandwidths \emph{only}. This lower bound on the bandwidth ensuring convergence depends solely on the maximum Euclidean norm of the samples, denoted by $\|x\|_{\max}$, and the kernel profile function $k$. It should be noted that the author were not initially aware of the more recent works \cite{YT2, YT}, where the restriction on the bandwidth being large was first lifted for \textit{real analytic} kernels, and later for a broad class of kernels that are bounded, continuous, convex, have Lipschitz continuous gradients, and are \textit{subanalytic} (cf. Assumption 4 on P.6 of \cite{YT}).\textit{ In this specific sense}, namely subanalytic kernels need not be positive (semi) definite and that bandwidths need not be large, Theorem~2 in \cite{YT} is more general than Theorem~\ref{mainTheo} in our paper.

However, the proof techniques and kernel classes considered in \cite{YT, YT2} and in the present paper are \emph{not comparable}. While \cite{YT} relies on subanalytic function theory and \L{}ojasiewicz inequalities, our analysis focuses on radially symmetric and strictly positive definite kernels generated by completely monotone profile functions. In particular, our framework allows kernel profiles $k$ that are\textit{ completely monotone} (see Section \ref{scn:completely monotone functions}), but \textit{not subanalytic}, an assumption these authors used in Theorem 2 in \cite{YT}.

Indeed, if $g$ is a Bernstein function, i.e. a non-negative function with a completely monotone first derivative, then the profile $k(r):=e^{-g(r)}$ is completely monotone by Lemma~3.4(i) in \cite{Merkle2014CompletelyMonotone}. Taking $g(t)=t^{\alpha}$ with \textit{irrational} $\alpha\in(0,1)$ yields a completely monotone kernel profile $k(t)=e^{-t^{\alpha}}$, which defines a radially symmetric positive definite kernel via the result of Schoenberg, see Proposition \ref{prop:Schonbergh}. On the other hand, since $\alpha$ is irrational, the function $t\mapsto t^{\alpha}$ at $0$ is not subanalytic (see, e.g., \cite{LeLoi_NotSubanalytic}, \cite{LeLoi_OminimalNotes}); consequently, $k(t)=e^{-t^{\alpha}}$ cannot be subanalytic either, as otherwise $-\log k(t)=t^{\alpha}$ would be subanalytic by closure of the (globally) subanalytic category under analytic ($-\log(.)$) composition (see, e.g., Ex. 1.11 from \cite{Coste2000}). Such kernels therefore fall\textit{ outside} the scope of \cite{YT, YT2}, while being naturally covered by the assumptions of this paper.

\subsection{Clustering at large bandwidths for certain non-Gaussian kernels.}
The experiments presented in Section \ref{scn:experiments} show that the use of large bandwidths
is not necessarily problematic solely because of the large bias it introduces
into the estimation, as noted in \cite{Gh1}, Section~4, p.~8.
Depending on the kernel choice and the underlying data geometry,
meaningful clustering behavior can still be observed.

In particular, for well-separated synthetic datasets,
certain non-Gaussian kernels (such as Laplace-type or Cauchy-type profiles)
can preserve multiple basins of attraction even at large bandwidths,
whereas Gaussian kernels may collapse to a single basin.
This behavior is explained theoretically at a heuristic level
in Section~\ref{scn:experiments}, Lemma \ref{lem:largeBW-globalmean} and Lemma \ref{lem:largeBW-powerlaw-local}, and then in the experiments.
It is noteworthy that the same heuristic mechanism applies
to a broader class of radially symmetric, positive definite kernels.

Given the above discussions, the main message of the paper is that it is possible to have \textit{simultaneous convergence} and \textit{reasonable clustering performance} relative to Gaussian collapse in the same bandwidth regime while using the MS algorithm at \textit{large bandwidths}.

\subsection{How the paper is organized:} Section \ref{scn:Mathematical description of the MS Algorithm} described the MS algorithm from the scratch. Sections \ref{scn:Main convergence theorems and steps of the proof of convergence}, is written to \textit{revisit} the proof strategy of the convergence theorem in \cite{Gh1}: this section is \textit{optional} and can be skipped, but for completeness, we add it in order to motivate the reader with the idea of the proof of the convergence theorem in \cite{Gh1}, so that we can use the same to prove our main Theorem \ref{mainTheo}. For the complete statement of the Theorem \ref{mainTheo}, the reader can go to Section \ref{scn:statement of the main theorem}. For the proof, we defined and stated and proved a few properties of \textit{completely monotone functions} in Section \ref{scn:completely monotone functions}, then we give the proof of Theorem \ref{mainTheo} in Section \ref{scn:details of the proof}. Section \ref{scn:experiments} gives us experiments with different datasets, with both Gaussian and Laplace kernels as an example of non-Gaussian kernels. 

\subsection{Acknowledgements} This research was supported by funding from Research Foundation– Flanders (FWO) via the Odysseus II programme no. G0DBZ23N. The author acknowledges the discussions with Praneeth Vepakomma (MBZUAI) on the experiments, and also the suggestions by the two anonymous referees to improve this manuscript.

\section{Mathematical description of the MS Algorithm}\label{scn:Mathematical description of the MS Algorithm}

\noindent In this introductory section, we will use the notations and terminologies used in \cite{CP}. The readers familiar with the basics can skip this section. In the Euclidean case, the algorithm is introduced below for completeness: let $\{x_1,x_2,...x_n\}$ be a set of $n$ independent data points in 
$ \ddim $ drawn from a distribution with an unknown pdf. For a radially symmetric, positive definite kernel we write the kernel density estimator (KDE) of this unknown pdf as: 

\[\widehat{f}_{h,k}(x) := c_N \sum_{i=1}^{n}k \left( || \frac{x-x_i}{h} ||^2 \right ).\] 

The function $k:[0, \infty)\to [0, \infty)$ is non-increasing, piecewise continuous, integrable over $[0, \infty),$ and is called the \textit{profile} of the kernel in use. Here 
$c_N$ is a normalization constant expressed as $c_N = \frac{c_{k, d, h}}{n}$, where $c_{k, d, h}$ is a constant depending on the dimension $d,$ and the bandwidth $h,$ ensuring that 
\[\int_{\ddim}\widehat{f}_{h,k}(x)dx=  \int_{\ddim}c_N \sum_{i=1}^{n}k \left( \norm{ \frac{x-x_i}{h} }^2 \right )dx =   \int_{\ddim} \frac{c_{k, d, h}}{n} \sum_{i=1}^{n}k \left( \norm{ \frac{x-x_i}{h} }^2 \right )dx = 1\]

as the total probability must be one.\\

\noindent Please note that the KDE $\widehat{f}_{h,k}(x)$ is written to stress the dependence on the bandwidth $h$ and the kernel $k$, but $h$ and $k$ will be often dropped not to clutter the notations. Let us also mention here that the bandwidth $h$ adjusts the size of the kernel.\\

\noindent Assuming that $k$ is differentiable with derivative $k'$, taking the gradient of $(1)$ w.r.t. $x$ yields:

\begin{equation}
    \nabla \widehat{f_{h,k}}(x)= \frac{2c_N}{h^2} \sum_{i=1}^{n} g(||\frac{x-x_i}{h}||^2) \left[ \underbrace{  \overbrace{  \frac{ \sum_{i=1}^{n} x_i g(||\frac{x-x_i}{h}||^2)  }{\sum_{i=1}^{n}  g(||\frac{x-x_i}{h}||^2)  }  }^{\text{Weighted mean}}   - x }_\text{Mean Shift vector}   \right]
\end{equation}

\noindent where $g=-k'$. The first term in the above equation is proportional to the density estimate at $x$ using kernel $G(x,y)$ with a profile of the form $c_N' g(|| \frac{x-y}{h} ||^2)$, where the second term is the difference of two terms : (i) the weighted mean of $x_i, i=1 \dots n,$ weighted by $g(||\frac{x-x_i}{h}||^2), g=-k',$ and is written as: $\frac{ \sum_{i=1}^{n} x_i g(||\frac{x-x_i}{h}||^2)  }{\sum_{i=1}^{n}  g(||\frac{x-x_i}{h}||^2)  }$(the first term in the overbrace) and (ii) the point $x.$ Thus the second term in the bracket above is called the mean shift vector (in the underbrace), $m_{h,g}(x)$, and hence the above equation can be written in the form:\\

\begin{equation}
    \nabla \widehat{f_{h,k}}(x) = \nabla \widehat{f_{h,g}}(x) \frac{2c_N}{h^2 c_N'}m_{h,g}(x)
\end{equation}

\noindent The above expression indicates that the MS vector computed with bandwidth $h$ and profile $k$ is proportional to the normalized gradient density estimate obtained with the profile  (normalization is done by density estimate with profile ). Therefore, the MS vector always points toward the direction of the maximum increase in the density function.\\

\noindent The modes of the estimated density function $\widehat{f_{h,k}}$ are located at the zeros of the gradient function $\nabla \widehat{f_{h,k}}=0$, i.e., . Equating $(2)$ to zero reveals that the modes of the estimated pdf are fixed points of the following function:

\begin{equation}
    m_{h,g}(x) + x =  \frac{ \sum_{i=1}^{n} x_i g(||\frac{x-x_i}{h}||^2)  }{  \sum_{i=1}^{n}  g(||\frac{x-x_i}{h}||^2)    }
\end{equation}

\noindent The MS algorithm initializes the mode estimate sequence to be one of the observed data. The mode estimate $y_j$  in the $j$'th iteration is updated as:

\begin{equation}\label{mode estimate sequence}
    y_{j+1}=m_{h,g}(y_j) + y_j =  \frac{ \sum_{i=1}^{n} x_i g(||\frac{y_j -x_i}{h}||^2)  }{  \sum_{i=1}^{n}  g(||\frac{y_j -x_i}{h}||^2)    }
\end{equation}

\noindent The MS algorithm iterates this step until the norm of the difference between two consecutive mode estimates is less than some predefined threshold. Typically $n$  instances of the MS algorithm are run in parallel, with the $i$'th instance initialized to the $i$'th data point.

\section{Main steps in convergence theorem in \cite{Gh1} and its proof}\label{scn:Main convergence theorems and steps of the proof of convergence}

\nd In describing our new result and other relevant results, we will closely follow the exposition in \cite{Gh1}, so the reader is strongly suggested to read that paper along with this one, although we try our best to make this article self-sufficient. With that in mind, we will closely examine the proof strategy of the main Theorem \ref{mainTheo} in \cite{Gh1}, as the strategy will be used to prove the main Theorems 1 and 2 of the present article.\\

\noindent Theorem 1 in \cite{Gh1} states:\\

\noindent \textbf{Theorem 1 and Lemma 10 from \cite{Gh1}:} For radially symmetric Gaussian kernels $K(x,y)= C_h exp(-|| x-y||^2/h^2)$ which is given by the profile $k(r):=e^{-r/2}$, the mode estimate $\{y_j\}$ given in \ref{mode estimate sequence} converges for all $h > max_{1\le i \le n}||x_i||$.\\

\noindent The main steps of the proof are as follows \label{Steps}:\\

\noindent \textbf{Step 1:} Let $k(r)=exp(-r/2)$ be the profile of the Gaussian kernel. There exists $h_0 > 0$ big enough so that, for each $h > h_0$, the Hessian of $\widehat{f_{h,k}}$, denoted by $Hess(\widehat{f_{h,k}})$ has full rank (nonsingular) at all the stationary points of $\widehat{f_{h,k}}$, i.e. at each point $x^{*}$ where $\nabla \widehat{f_{h,k}}(x^{*})=0$. This is essentially the Lemma  10 from \cite{Gh1}, where his calculation showed that:
$h_0= max_{1 \leq i \leq n} {||x_i||}=: \norm{x_{max}}$ . \\

\noindent \textbf{Step 2:} If the Hessian matrix of any $\mathcal{C}^2$ function at its stationary points is of full rank/nonsingular, the stationary points are isolated. This is essentially the Lemma 2 from \cite{Gh1}, easily proved using the inverse mapping theorem. \textbf{Note that the converse is clearly not true: one can have singular Hessian while having the stationary points of the gradient isolated:} just take $\psi:\cl\to \cl:= \psi(t)=t^4, Hess \,{\psi}(0)=0$ in one dimension to find a counterexample. \\

\noindent \textbf{Step 3:} Let $x_i \in \ddim, i=1,2...n.$ Assume that the stationary points of the estimated pdf are isolated. Then the mode estimate sequence $\{y_j\}$ (cf. Equation \ref{mode estimate sequence}) defined in Section \ref{scn:Mathematical description of the MS Algorithm} converges. This is essentially the Theorem 1 (main theorem) in \cite{Gh1}.\\

\noindent \textbf{An important takeaway from the three steps above:}\\

\noindent Note carefully that, only one of the three steps above, namely \textbf{Step 1}, depends on the fact that the kernel is Gaussian, or equivalently, its profile $k$ satisfies $k(r):=e^{-r/2}$. This tells us that, if we can find kernel profiles $k$ so that the Hessians of the KDE: $Hess(\widehat{f_{h,k}})$ are of full rank at the stationary points of the KDE $\widehat{f_{h,k}}$, then  \textbf{Step 2} follows through, and thus \textbf{Step 3} follows through as well, and we end up with convergence of the mode estimate $y_j$ in equation \ref{mode estimate sequence}. The main Theorem \ref{mainTheo} does \textit{precisely} that, and generalizes the proof strategy to \textit{arbitrary} radially symmetric, positive definite kernels.

\section{Statement of the main theorem}\label{scn:statement of the main theorem}

\noindent To prove Theorem \ref{mainTheo} and its precursor Lemma \ref{lem:convergence w positive definite kernel} mentioned in the previous section, we will need two theorems from \cite{Fass}, Page~16, Theorems~2.5.2 and~2.5.3, which characterize a special type of function called \textit{completely monotone functions}. In the following, we state the definition and relevant results on them.\\

\begin{definition}[\textbf{Positive semidefinite and positive definite square matrices}]\label{Positive and strictly positive definite square matrices}
An \( n \times n \) Hermitian matrix \( A \) is called \emph{positive semidefinite (PSD)} if 
\(\langle A v, v \rangle \ge 0\) for all \(v \in \mathbb{R}^n\);
it is \emph{positive definite (PD)} if \(\langle A v, v \rangle > 0\) for all \(v \ne 0\).
\end{definition}

\begin{definition}[\textbf{Positive and strictly positive definite functions}]\label{Positive and strictly positive definite kernels}
A map \(K:\ddim \times \ddim \to [0,\infty)\) is called \textit{positive definite} if, for any finite set of distinct points 
\(\{x_i: 1\le i \le n\}\), the matrix 
\(\big[k(x_i, x_j)\big]_{1\le i,j \le n}\) is positive semidefinite. 
It is called \textit{strictly positive definite} if the matrix 
\(\big[k(x_i, x_j)\big]_{1\le i,j \le n}\) is positive definite.
\end{definition}

\nd The main theorem essentially tells us that the MS algorithm with arbitrarily, radially symmetric, strictly positive kernel converges when the bandwidth is sufficiently large. 

\begin{theorem}\label{mainTheo}
Consider the MS algorithm with the radially symmetric, strictly positive definite kernel $K(x,y)  :=k(\norm{x-y}^2),$ given by the kernel profile function $k.$ Then there's $h_0 > 0$  depending only on the kernel profile $k$ and its first two derivatives and the maximum norm $\norm{x_{max}}:= max_{1 \le i \le n}\norm{x_i}$ so that the mode estimate sequence $\{y_i=y_i(h)\}$ in Equation \eqref{mode estimate sequence} of the MS algorithm described in Section \ref{scn:Mathematical description of the MS Algorithm} converges $\forall h > h_0$. More specifically, $h_0$ is given by:

\begin{equation}\label{eqn:equation-main-Theorem}
  {-2q_0 \frac{k''(q_0)}{k'(q_0)}= 1, \text{where } q_0 := \frac{4\norm{x_{max}}^2}{h_0^2}}.  
\end{equation}

\end{theorem}

\nd Following the three steps mentioned in the beginning of section \ref{scn:Main convergence theorems and steps of the proof of convergence} (see \ref{Steps}), proving the lemma below would prove the main theorem.

\begin{lemma}\label{lem:convergence w positive definite kernel}
Consider the MS algorithm with the radially symmetric, positive definite kernel $K(x,y) :=k(\norm{x-y}^2)$. Assume that $K$ is a \textbf{strictly} positive definite kernel with profile $k.$ Then there's $h_0 > 0$ depending only on the kernel $k$ and its first two derivatives, and the maximum norm $\norm{x_{max}}:= max_{1 \le i \le n}\norm{x_i}$ so that $Hess(\widehat{f_{h,k}})(x)$ has full rank at every stationary point $x^\ast$ of $\widehat f_{h,k}$.

More specifically, $h_0$ is given by:

$${-2q_0 \frac{k''(q_0)}{k'(q_0)}= 1, \text{where } q_0 := \frac{4\norm{x_{max}}^2}{h_0^2}}.$$
\end{lemma}

\nd The next section is written as a \textit{self-sufficient overview} of alternate characterization of radially symmetric, positive definite kernel matrices using completely monotone functions and characterization of the latter using certain types of positive Borel measures. The key results from these will be used in Sction \ref{scn:details of the proof} to prove Lemma \ref{lem:convergence w positive definite kernel} and thus the main Theorem \ref{mainTheo}.\\

\section{Completely monotone functions and their connections to radially SPD kernel matrices}\label{scn:completely monotone functions}

\begin{definition}[\textbf{Completely monotone functions}]\label{Completely monotone functions}
\noindent A function \(k:[0, \infty) \to \cl\) is called \emph{completely monotone} if it satisfies the following properties:
\begin{enumerate}
\item \(k \in \mathcal{C}([0,\infty))\);
\item \(k \in \mathcal{C}^{\infty}((0,\infty))\);
\item \(k\) and its derivatives \(k^{(l)}\) satisfy the \emph{alternating sign} condition:
\((-1)^{l}k^{(l)}(r)\geq 0\) for all \(l \in \mathbb{N}\) and all \(r>0\).
\end{enumerate}
\end{definition}

\noindent Examples of $k$ include \(r \mapsto C,\, C\ge 0\); \(r\mapsto e^{-rt},\, t\ge 0\); \(r\mapsto (r+1)^{-t},\, t\ge 0\).
It can be proved that sums, products, and positive scalar multiples of completely monotone functions are also completely monotone.\\

\noindent Next, we restate two classical characterizations of completely monotone functions following~\cite{Fass}:
\begin{itemize}
  \item one relates them to Laplace transforms of finite non-negative Borel measures (Hausdorff--Bernstein--Widder theorem),
  \item the other relates them to positive-definite, radially symmetric kernels (Schoenberg characterization).
\end{itemize}

\begin{proposition}[\textbf{Hausdorff--Bernstein--Widder representation}]\label{prop:HBD thm}
\noindent A function \(k:[0,\infty)\to\cl\) is completely monotone if and only if it is the Laplace transform of a finite non-negative Borel measure~\(\mu\) on~\([0,\infty)\); that is,
\[
k(r)= \mathcal{L}\mu(r)=\int_{0}^{\infty} e^{-rt}\,d\mu(t).
\]
\end{proposition}

We assume this strictness throughout the remainder of the paper. A proof of this result can be found in~\cite{Fass}, Theorem~2.5.2.

\begin{proposition}[\textbf{Schoenberg characterization and strictness criterion}]\label{prop:Schonbergh}
A function \(k\) is completely monotone on \([0,\infty)\) if and only if 
\(\Phi(x):=k(\|x\|^2)\) is positive semidefinite (cf. Definition \ref{Positive and strictly positive definite kernels}) and radial on \(\mathbb{R}^d\) for all \(d\in\mathbb{N}\).
Moreover, \(\Phi\) is \emph{strictly} positive definite on any set of distinct points if and only if the representing measure~\(\mu\) in~\eqref{prop:HBD thm} satisfies \(\mu((0,\infty))>0\) (equivalently, \(k\) is not constant).
\end{proposition}

A proof of this result is given in~\cite{Fass}, and an expanded version including the strictness criterion is provided later in Proposition~\ref{prop:modified Schonberg}.

\begin{remark}
\textbf{Note.} In this context and related machine-learning or statistics literature, the terms \textit{non-negative definite} and \textit{positive definite} are often used in place of 
\textit{positive semidefinite} and \textit{positive definite}, respectively, while in pure mathematics these distinctions are kept explicit.
\end{remark}

\noindent Although the previous two propositions form the building blocks of our main result, they need to be modified slightly in order to fit into the hypotheses of, and hence prove, Theorem~1. These modified versions are stated below---the author expresses his gratitude and acknowledgement to Prof.~Iosif Pinelis for his contribution to these versions. In what follows,~\(\delta_0\) stands for the Dirac measure concentrated at~0.

\begin{proposition}[\textbf{Modified Proposition \ref{prop:HBD thm}}]\label{prop:modified HBD}
A \textbf{non-constant} function \(k:[0,\infty)\to\mathbb{R}\) is completely monotone if and only if it is the Laplace transform of a finite non-negative Borel measure~\(\mu\) on~\([0,\infty)\) not of the form \(c\delta_0,\ c>0\); that is,
\begin{equation}
k(r)= \mathcal{L}\mu(r)=\int_{0}^{\infty} e^{-rt}\,d\mu(t),
\end{equation}
where \(\mu\) is not of the form \(c\delta_0,\ c>0\).
\end{proposition}

\begin{proof}
Let us actually prove a bit more.\\[3pt]
\noindent A function \(k:[0,\infty)\to\mathbb{R}\) is non-constant and completely monotone if and only if 
\begin{equation}
k(r)=\int_{0}^{\infty} e^{-rt}\,d\mu(t)\quad \forall r\geq 0,
\end{equation}
with \(\mu\ne c\delta_0\) for any real \(c\). \\[4pt]
\noindent \textbf{The ``only if'' part:} Suppose \(k\) is non-constant and completely monotone. Then (6) holds with some (unique) finite measure~\(\mu\). If \(\mu=c\delta_0\) for some real \(c\), then \(k\) is the constant~\(c\), contradicting non-constancy. Hence \(\mu\ne c\delta_0\) for any~\(c\).\\[4pt]
\noindent \textbf{The ``if'' part:} Suppose (6) holds with \(\mu\ne c\delta_0\) for some~\(c\). Then \(k\) is completely monotone and \(\mu((0,\infty))>0\), whence
\begin{equation}
k'(r)=-\int_{0}^{\infty} te^{-rt}\,d\mu(t)<0\quad \forall r\ge0,
\end{equation}
so that \(k\) is not constant. 
\end{proof}

\begin{proposition}{(\textbf{Modified Proposition \ref{prop:Schonbergh}})}\label{prop:modified Schonberg}
A \textbf{non-constant} function \(k\) is completely monotone on \([0,\infty)\) if and only if \(\Phi(x):=k(\|x\|^2)\) is \textbf{strictly} positive definite and radial on \(\mathbb{R}^d\) for all \(d \in \mathbb{N}\).
\end{proposition}

\noindent The proof is based on:

\begin{lemma}\label{lem:char fn strictly positive definite}
The characteristic function of any probability distribution on \(\cl^d\) with uncountable support is strictly positive definite. 
\end{lemma}

\noindent Assume the lemma for the moment, which will be proved at the end of the theorem. Using Lemma \ref{lem:char fn strictly positive definite}, let us now prove: \\[2pt]
\noindent \textbf{The "only if" part of Proposition \ref{prop:modified Schonberg}:} Suppose \(k\) is non-constant and completely monotone. Then, using the notation \(\Phi\) as in Proposition  above, we have
\[
\Phi(x)=\int_0^\infty e^{-t\|x\|^2}\,d\mu(t)\quad\forall x\in\cl^d.
\]
For each real \(t>0\), the function \(x\mapsto e^{-t\|x\|^2}\) is the characteristic function of a Gaussian distribution on~\(\mathbb {R}^d\), namely the Gaussian distribution with mean~0 and covariance~\(\Sigma=t\,I_d\). Hence, by Lemma \ref{lem:char fn strictly positive definite}, the function \(x\mapsto e^{-t\|x\|^2}\) is strictly positive definite. Also, by the proof of Proposition \ref{prop:modified HBD}, \(\mu\) is not of the form \(c\delta_0,\ c > 0\), so the measure \(\mu\) is not concentrated at zero. Thus \(\mu((0,\infty))>0\). Consequently, by the expression of \(\Phi\) above and integrating with respect to~\(\mu\), we conclude that \(\Phi\) is strictly positive definite. \\[4pt]
\noindent \textbf{The "if" part of Proposition \ref{prop:modified Schonberg}:} Suppose \(\Phi\) is strictly positive definite. Then, by Proposition \ref{prop:Schonbergh} above, \(k\) is completely monotone. So, (6) holds and hence 
\[
\Phi(x)=\int_0^\infty e^{-t\|x\|^2}\,d\mu(t)\quad\forall x\in\mathbb {R}^d.
\]
Since \(\Phi\) is strictly positive definite, we must have \(\mu((0,\infty))>0\); for otherwise, if \(\mu((0,\infty))=0\), then
\(\Phi(x)=\int_{\{0\}} e^{-t \|x\|^2}\,d\mu(t)=\mu(\{0\})\),
which cannot be strictly positive definite. As \(\mu((0,\infty))>0\), by equation~(7) for \(k'\) we have \(k'<0\), hence \(k\) is non-constant. \\[4pt]
\noindent \textbf{Proof of Lemma \ref{lem:char fn strictly positive definite}.} Let \(f\) be the characteristic function of a random vector \(Z\) in \(\mathbb {R}^d\) with uncountable support. To obtain a contradiction, suppose that \(f\) is not strictly positive definite. Then, for some \(k\in\mathbb{N}\), some complex \(a_1,\dots,a_k\) not all zero, and some pairwise distinct \(x_1,\dots,x_k\in\mathbb{R}^d\),
\[
E\Big|\sum_{j=1}^k e^{i x_j\cdot Z}a_j\Big|^2
=\sum_{j,l=1}^k f(x_j-x_l)a_j\bar a_l=0,
\]
where \(\cdot\) denotes the dot product. Hence \(h(z):=\sum_{j=1}^k e^{i x_j\cdot z}a_j=0\) for all \(z\) in the uncountable support of the distribution of~\(Z\). Since \(h\) is entire, it can have at most countably many zeros; therefore \(h(z)=0\) for all \(z\in\mathbb{R}^d\), implying \(a_1=\dots=a_k=0\), a contradiction. \\

\noindent Following Propositions \ref{prop:modified HBD}, \ref{prop:modified Schonberg} above, we can state for our purposes:

\begin{proposition}[\textbf{Kernel matrix of a radially symmetric, strictly positive definite kernel}]\label{prop:kernel matrix of a rad symm, SPSD kernel}
Let the radially symmetric, strictly positive definite kernel \(K\) be given by a profile \(k\). Then
\(k(r)= \int_{0}^{\infty} e^{-rt}\,d\mu(t)\) for all \(r\geq 0\), and thus for all \(x_i,x_j \in \ddim\), the corresponding strictly positive definite kernel matrix is
\[
K = \int_{0}^{\infty}\!\!\begin{bmatrix}
1 & e^{-t\|x_1-x_2\|^2} & \dots & e^{-t\|x_1-x_n\|^2}\\
e^{-t\|x_2-x_1\|^2} & 1 & \dots & e^{-t\|x_2-x_n\|^2}\\
\vdots & \vdots & \ddots & \vdots\\
e^{-t\|x_n-x_1\|^2} & \dots & \dots & 1
\end{bmatrix} d\mu(t),
\]
for some finite positive Borel measure \(\mu\) not of the form \(c\delta_0,\ c>0.\)
\end{proposition}
\begin{remark}[\textbf{Examples of completely monotone profiles used in Section~\ref{scn:experiments}}]\label{rmk:examples-CMF-experiments}
The kernel profiles
\[
k_{\mathrm{Lap}}(q):=e^{-\lambda\sqrt q}\qquad(\lambda>0),
\qquad\text{and}\qquad
k_{\alpha}(q):=e^{-\lambda q^\alpha}\qquad(\lambda>0,\ \alpha\in(0,1])
\]
are completely monotone on $(0,\infty)$.
Indeed, $q\mapsto q^\alpha$ is a Bernstein function (i.e. a non-negative function with a completely monotone first
derivative) for $\alpha\in(0,1]$, and if $\phi$ is Bernstein then
$q\mapsto e^{-\lambda \phi(q)}$ is completely monotone for every $\lambda>0$ by Lemma~3.4(i) in \cite{Merkle2014CompletelyMonotone}; hence $k_\alpha$ is completely monotone.
The Laplace profile is the special case $\alpha=\tfrac12$.
Consequently, by Schoenberg's characterization (cf. Proposition \ref{prop:modified Schonberg}), each $k_\alpha$ (and $k_{\mathrm{Lap}}$) generates a radially symmetric positive definite kernel on $\mathbb{R}^d$ for every $d$ via $K(x,y)=k(\|x-y\|^2)$.
\end{remark}

\section{Details of the proof of the Lemma \ref{lem:convergence w positive definite kernel} and hence of the Main Theorem \ref{mainTheo}}\label{scn:details of the proof}

\noindent Assume, as in the hypothesis of Theorem 1, we are working with a radially symmetric, strictly positive definite kernel with non-constant, completely monotone profile (Theorem 6) $k,$ arising as the Laplace transform of a finite, positive Borel measure $\mu \ne c\delta_0, c >0.$ So let us start by computing the gradient and Hessian of the KDE; in what follows, we will abbreviate the notation $f_{h,k}$ by $f$ where there are no risk of confusion. Recall that $k'<0, k'' > 0$ as they will be used in the subsequent calculations. Below we denote by $I_{d \times d}$ the identity matrix of dimensions $d\times d.$

\nd Recall the definition of KDE in the beginning of section \ref{scn:Mathematical description of the MS Algorithm}, the normalization constant $c_N$ depending on the kernel $k,$ the number $n$ of data points, the dimension $d,$ and the bandwidth $h.$ 

\begin{equation}\label{eqn:KDE}
  \hat{f}(x)\equiv \hat{f}_{h,k}(x) = c_N \sum_{i=1}^{n}  k \left (  \norm{ \frac{x-x_i}{h}}  ^2  \right )   
\end{equation}

\nd \textbf{Note:} When the kernel is Gaussian, i.e. the profile $k(r):= exp(-r/2), c_N= \frac{1}{(\sqrt{2\pi})^{d}nh^d}.$

\noindent Taking gradient w.r.t. $x$ of equation (\ref{eqn:KDE}) above,

$$\nabla \hat{f}(x)=2c_N \sum_{i=1}^{n}  \left (\frac{x- x_i}{h^2} \right )\kdxxih  = 2c_N \sum_{i=1}^{n} \left (\frac{x- x_i}{h^2} \right ) \intnive t \hspace{1mm} \gauss d\mu(t)$$

\noindent Taking gradient once more, we get:\\
\begin{equation*}
\begin{split}
      H(x) & = Hess  \hspace{1mm} \hat{f}(x) \\
 & = \frac{2c_N}{h^2} \sum_{i=1}^{n} \left[ I_{d \times d} \left(  \kdxxih + 2\frac{\covxxi}{h^2} \kddxxih   \right)\right] \\
 & = \frac{2c_N}{h^2}\sum_{i=1}^{n} \intnive \left (t I_{d \times d} + \frac{2t^2}{h^2} (x - x_i)(x - x_i)^{T}  \right ) \hspace{1mm} \gauss d\mu(t). 
\end{split}
\end{equation*}
 
\noindent Define, following \cite{Gh1}'s notations (P.5, section 4.1), the following two functions, respectively:

$C: \ddim \to \cl $ by:

$$C(x):= -2 \sumn\kdxxih  = \intnive 2t    \left[ \sum_{i=1}^{n} \gauss  \right ] d\mu(t)$$

\nd \textbf{Quick check:} When $k(r):= exp(-r/2), k'(r)=  -\frac{1}{2}exp(-r/2),$ so in this case the expression of $C(x)$ for general kernel matches up with the one for the Gaussian kernel, given in section 4.\\

\nd Next define $A: \ddim \to \cl^{d \times d}$ by:

$$A(x):= 4 \sumn \covxxi \kddxxih =4 \intnive t^2 \left [(x-x_i)(x-x_i)^{T} \gauss  \right ] d\mu(t)$$

\nd \textbf{Quick check:} When $k(r):= exp(-r/2), k''(r)=  \frac{1}{4}exp(-r/2),$ so in this case the expression of $A(x)$ for general kernel matches up with the one for the Gaussian kernel, given in section 4.\\

\noindent With the above functions and their notations, let's take another look at the Hessian $H(x):$

\begin{equation}
  \begin{split}
        H(x) 
  & = \frac{2c_N}{h^2} \sum_{i=1}^{n} \left[I_{d \times d} \left(  \kdxxih + 2\frac{\covxxi}{h^2} \kddxxih   \right)\right] \\
  & = \frac{c_N}{h^2} { \sum_{i=1}^{n} 2 \kdxxih } (I_{d \times d}) + \frac{4c_N}{h^2}{ \sum_{i=1}^{n}\covxxi \kddxxih   }\\
  & = -\frac{c_N}{h^2} \underbrace{ \sum_{i=1}^{n} -2 \kdxxih }_\text{function C(x) } (I_{d \times d}) + \frac{c_N}{h^4}\underbrace{ 4\sum_{i=1}^{n}\covxxi \kddxxih   }_\text{ map A(x)}
  \end{split}
\end{equation}

\nd So with the above decomposition of the Hessian $H(x),$ we can write:

$$  \boxed{  H(x) = - c_N \frac{C(x)}{h^2} I_{d \times d}  + c_N \frac{A(x)}{h^4}  }$$

\noindent Now, we need to show that, for sufficiently large bandwidth $h$, $H(x)$ is of full rank at every $x$. So if possible, assume that, $H$ is not of full rank $\equiv$ $H$ has a zero eigenvalue $v \in \cl^{D \times 1}$ so that $H(x)v =0 \in \cl^{D \times 1}$. Using the above expression for $H$ above,\\

\begin{equation}
    \begin{split}
        \left [- c_N \frac{C(x)}{h^2} I_{d \times d}  + c_N \frac{A(x)}{h^4} \right] v & = 0 \\
      \implies A(x) v & = h^2 C(x) v  \\
   \implies \boxed{4 \sumn \covxxi \kddxxih v} &=  \boxed{h^2 (-2 \sumn\kdxxih)v}\\
    &= h^2 \intnive 2t    \left[ \sum_{i=1}^{n} \gauss  \right ] d\mu(t) v
  \end{split}
\end{equation}

\nd \textbf{We state that the above equation is fundamental to proving Lemma \ref{lem:convergence w positive definite kernel}, thus Theorem \ref{mainTheo}, as we will see in the next section.}

\noindent Let us abbreviate notation and write the above as:\\
\begin{equation}
 4\sum_{i=1}^{n}(x-x_i)(x-x_i)^{T} J_i(h) v = - 2 h^2 \sum_{i=1}^{n} I_i(h)v
\end{equation}

\noindent Where:

$$I_i(h) := \intnive t   \gauss d\mu(t) = -\kdxxih$$ and 
$$ J_i(h):= \intnive t^2   \gauss d\mu(t) = \kddxxih$$

\nd By taking (Euclidean) norm on both sides:

\begin{equation}
   4\norm{ \sumn \covxxi  J_i(h)v} = 2h^2 \norm{ \sumn I_i(h)v} , v\in \ddim
\end{equation}

\noindent Note that the kernel profile $k$ is smooth (infinitely differentiable) thanks to Proposition \ref{prop:modified HBD}. Then as $h \to \infty$, each $I_i(h) \to \intnive t d\mu(t) = - k'(0)$ and each $J_i(h)\to \intnive t^2 d\mu(t) =  k''(0).$ So the above equation is a contradiction as $h \to \infty$. So there exists a sufficiently large $h_0 > 0$ so that $\forall h > h_0$, the equation (9) won't hold, meaning that Hessian $H(x)$ has no zero eigenvalue.  So it has full rank. This shows that for sufficiently large value of $h,$ Hess $H(x)$ has full rank, but it doesn't give us a precise lower bound for $h$ for $H(x)$ to be of full rank. This is something we will work right next.\\

\nd \textbf{Finding a more precise bound $h_0$:} Let us recall that $H(x)= 0 \implies \frac{A(x)}{h^2}v = C(x)v$ for some $v \in \ddim.$ So if we can find $h>0$ so that the maximum eigenvalue of $\frac{A(x)}{h^2}$ is less than $C(x),$ we are done. This is precisely what we do below. Take any $x=x^*$ that's a stationary point of $\hat{f}_{h,k}$, and since $x^*$ lies in the convex hull of the sample $\{x_1\dots x_n\}, \norm{x^*}\le \norm{x_{max}}$, so:  
\begin{equation}
    \begin{split}
       & \lambda_{max}\left(\frac{A(x^*)}{h^2}\right) \\
       & = \frac{\lambda_{max}(A(x^*))}{h^2}\\
       & = \frac{1}{h^2}\lambda_{max}\left(  4 \sumn \covxstxsti\kddxxihstar \right)\\
       & = \frac{4}{h^2}\lambda_{max}\left( \sumn \covxstxsti \kddxxihstar \right)\\
       & \le \frac{4}{h^2} \sumn \lambda_{max} \left(\covxstxsti \kddxxihstar \right)\\
       & = \frac{4}{h^2} \sumn \lambda_{max} (\covxstxsti) \kddxxih \\
       & = \frac{4}{h^2} \sumn \norm{\xst-x_i}^2 \kddxxihstar \\
       & \le \frac{4}{h^2} \sumn (2\norm{x_{max}})^2 \kddxxihstar (\text{Using }  \norm{x^*}\le \norm{x_{max}})\\
       & = \frac{16\norm{x_{max}}^2}{h^2} \sumn  \kddxxihstar
    \end{split}
\end{equation}

\nd So if we want to make sure that $\lambda_{max}( \frac{A(\xst)}{h^2}) < C(\xst),$ we can have that if we assume $h$ is so that:\\

\begin{equation}
    \begin{split}
      & \lambda_{max}( \frac{A(\xst)}{h^2}) < C(\xst)\\
      & \impliedby   \frac{16\norm{x_{max}}^2}{h^2} \sumn  \kddxxihstar < -2 \sumn\kdxxihstar \\
      & \iff \frac{h^2}{\norm{x_{max}}^2} > \frac{16 \sumn  \kddxxihstar }{ -2 \sumn\kdxxihstar} = -8  \frac{\sumn  \kddxxih }{ \sumn\kdxxihstar}
    \end{split}
\end{equation}

\nd Next note that since $k''$ is a decreasing function and $k'$ is an increasing function (since $k'''(r)\le 0$ and $k''(r)\ge 0$ for completely monotone $k$), this means:
\[
k''\!\left(\frac{4\norm{x_{max}}^2}{h^2}\right) \le \kddxxihstar,
\qquad
k'\!\left(\frac{4\norm{x_{max}}^2}{h^2}\right) \ge \kdxxihstar.
\]

\nd Therefore,
\[
n k''\!\left(\frac{4\norm{x_{max}}^2}{h^2}\right)
= \sumn k''\!\left(\frac{4\norm{x_{max}}^2}{h^2}\right)
\le \sumn \kddxxihstar,
\]

\nd \text{and}
\[
n k'\!\left(\frac{4\norm{x_{max}}^2}{h^2}\right)
= \sumn k'\!\left(\frac{4\norm{x_{max}}^2}{h^2}\right)
\ge \sumn \kdxxihstar.
\]

\nd Therefore,
\[
\frac{ k''\!\left(\frac{4\norm{x_{max}}^2}{h^2}\right)}
     { k'\!\left(\frac{4\norm{x_{max}}^2}{h^2}\right)}
\le
\frac{\sumn \left(\kddxxihstar\right)}
     {\sumn \left(\kdxxihstar\right)}.
\]
Combining this with equation (12), we see that a sufficient condition is

$$\frac{h^2}{\norm{x_{max}}^2} > -8 \frac{ k''\left(\frac{4\norm{x_{max}}^2}{h^2}\right)}{ k'\left(\frac{4\norm{x_{max}}^2}{h^2}\right)}$$

\nd For simplicity, if we define $q := \frac{4\norm{x_{max}}^2}{h^2},$ then the above equation becomes: 

$$\boxed{-2q \frac{k''(q)}{k'(q)}< 1, \text{where } q := \frac{4\norm{x_{max}}^2}{h^2}}$$

\nd \textbf{Check:} As $h$ increases, $q:= \frac{4\norm{x_{max}}^2}{h^2}$ decreases, so $k''(q) > 0$ increases since $k'''(q)< 0$. Moreover, as $q$ decreases, $k'(q) < 0$ increases since $k''(q) > 0,$ so $-k'(q) > 0$ decreases. Therefore $\frac{-2}{k'(q)}$ decreases. Putting everything together: $-2q \frac{k''(q)}{k'(q)}$ decreases. Finally using  $k\in \mathcal{C}^2([0,\infty)),$ we see that as $h \to \infty, q\to 0, k'(q)\to k'(0), k''(q)\to k''(0) \implies -2q \frac{k''(q)}{k'(q)}\to 0$ as well. So this means if we take $h_0$ any number so that for the corresponding $q_0:= \frac{4\norm{x_{max}}^2}{{h_0}^2}, -2q_0\frac{k''(q_0)}{k'(q_0)} < 1,$ then for every $h > h_0,$ the same condition will continue to hold. This gives us our desired $h_0.$ Note that in the case of the Gaussian kernel $k(q):= exp(-q/2)$, the condition reduces to $ q < 1 \iff h > 2\norm{x_{max}}$. This proves Lemma 5 and so the main Theorem 1.

\begin{remark}
 \nd As mentioned right before the conclusion section (P.8) of \cite{Gh1}, having a large bandwidth $h$ amounts to smoothing the pdf and low estimation variance at the cost of a large bias for mode estimation of the pdf. So, the above sufficient condition for convergence of MS with more general kernels, although informative, may not be practically desirable. However, our experiments in the next section show that even for bandwidths large enough, we can get the right number of clusters in some cases and if we use the right kernels, then number of basins of attractions can still be meaningful.
\end{remark}

\section{Experiments}\label{scn:experiments}

Experiments in this section validate (i) the theorem-prescribed $h_0$
logic where applicable, and (ii) the practical clustering behavior of kernels at large bandwidth; Laplace and Cauchy-type kernels are treated as out-of-the-theorem benchmarks because for these two kernels, the equation \eqref{eqn:equation-main-Theorem} defining $h_0$ has no finite solution, see subsection \ref{scn:non-existence-h0-Laplace-kernel} below. We also show by theoretical heuristics that Gaussian kernels are \textit{not }the ideal kernels to use for large bandwidths although they guarantee convergence of the MS algorithm \cite{Gh1}; see Lemma \ref{lem:largeBW-globalmean}. On the other hand, there are classes of non-Gaussian, radially symmetric, strictly positive definite kernels that can ensure good clustering performance even with large bandwidths, see Lemma \ref{lem:largeBW-powerlaw-local}, Cauchy-type kernel being of these class of kernels.

\subsection{Solving for $h_0$} In order to experimentally realize some instances of the main theorem \ref{mainTheo}, we solve the equation \ref{eqn:equation-main-Theorem} given by $-2 q_0 \frac{k^{\prime \prime}\left(q_0\right)}{k^{\prime}\left(q_0\right)}=1, \text { where } q_0:=\frac{4\left\|x_{\max }\right\|^2}{h_0^2}$ for $h_0$, for different positive definite kernels. We then use the solved $h_0$ to perform iterations of mean shift as given by $y_{j+1}=m_{h, g}\left(y_j\right)+y_j=\frac{\sum_{i=1}^n x_i g\left(\left\|\frac{y_j-x_i}{h}\right\|^2\right)}{\sum_{i=1}^n g\left(\left\|\frac{y_j-x_i}{h}\right\|^2\right)}$ over several standard datasets such as\textit{ Iris } and on synthetic datasets such as samples generated from \textit{isotropic Gaussians with separate means} to represent different clusters. Below, we note that while Equation \eqref{eqn:equation-main-Theorem} can be used to solve for a finite $h_0$ for Gaussian kernels, it is not the case for some other radially symmetric, strictly positive definite kernels like Laplace or Cauchy kernels, in which case we will use large finite bandwidths in our experiments.
\subsubsection{Gaussian Kernel}  The Gaussian kernel profile is given by $k(r):=e^{-r / 2}$, where the main theorem \ref{mainTheo} yields $h_0=2\norm{x_{max}},$ which will be used in the experiments.

\subsubsection{Laplace Kernel non-existence of a finite $h_0$.}\label{scn:non-existence-h0-Laplace-kernel}

Consider the Laplace kernel
\[
k(q)=e^{-\lambda \sqrt{q}}, \qquad \lambda>0 .
\]
A direct computation gives
\[
k'(q)=-\frac{\lambda}{2}q^{-1/2}e^{-\lambda\sqrt{q}}, 
\qquad
k''(q)=\Big(\frac{\lambda}{4}q^{-3/2}+\frac{\lambda^2}{4}q^{-1}\Big)e^{-\lambda\sqrt{q}} .
\]
Therefore,
\[
\frac{k''(q)}{k'(q)}
=-\frac12\,q^{-1}\bigl(1+\lambda\sqrt{q}\bigr),
\qquad\text{and hence}\qquad
-2q\,\frac{k''(q)}{k'(q)}=1+\lambda\sqrt{q}.
\]
The bandwidth equation \ref{eqn:equation-main-Theorem} in Theorem \ref{mainTheo}
\[
-2q_0\,\frac{k''(q_0)}{k'(q_0)}=1
\]
thus implies $\lambda\sqrt{q_0}=0$, and consequently $q_0=0$.
Recalling that $q_0=4\|x_{\max}\|^2/h_0^2$, this yields $h_0=\infty$. 

\subsubsection{Non-existence of a finite $h_0$ for Cauchy-type kernels.}
For the power-law Cauchy-type profile 
\[
k(q)=\frac{1}{1+q^\alpha}, \qquad 0<\alpha<1,
\]
a direct computation gives
\[
-\frac{2q k''(q)}{k'(q)}
=
2\alpha\,\frac{1-\alpha + q^\alpha}{1+q^\alpha}.
\]
Since $0<\alpha<1$, the right-hand side satisfies
\[
0 < -\frac{2q k''(q)}{k'(q)} < 2\alpha < 1
\quad \text{for all } q>0,
\]
and therefore the equation 
\(
-2q\,k''(q)/k'(q)=1
\)
admits no finite solution. Hence $h_0=+\infty$ in this case.

\medskip

\subsection{Failure case of Gaussian kernel, backed up by theory}

We note that although the MS mode estimate sequence converges in the case of the Gaussian kernel, it does not necessarily yield the right number of clusters at the values of
$h_0=2\norm{x_{max}}$ prescribed by the main theorem \ref{mainTheo} as shown in Figure \ref{fig:oneGauss} below. 

\begin{figure}[!htbp]
    \centering   \subfloat[\centering above optimal $h_0$]{{\includegraphics[width=5cm]{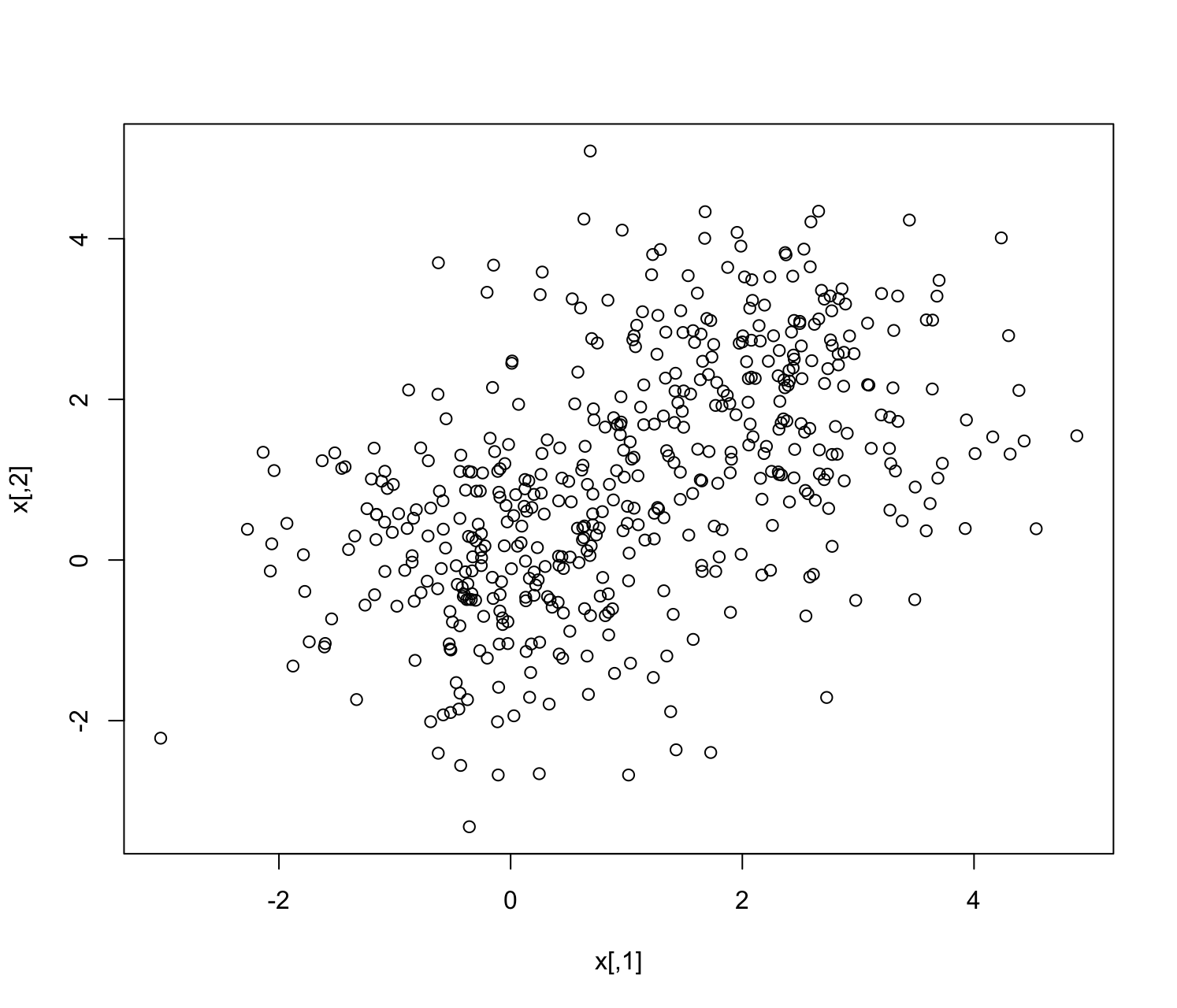} }}%
    \subfloat[\centering below optimal $h_0$]{{\includegraphics[width=5cm]{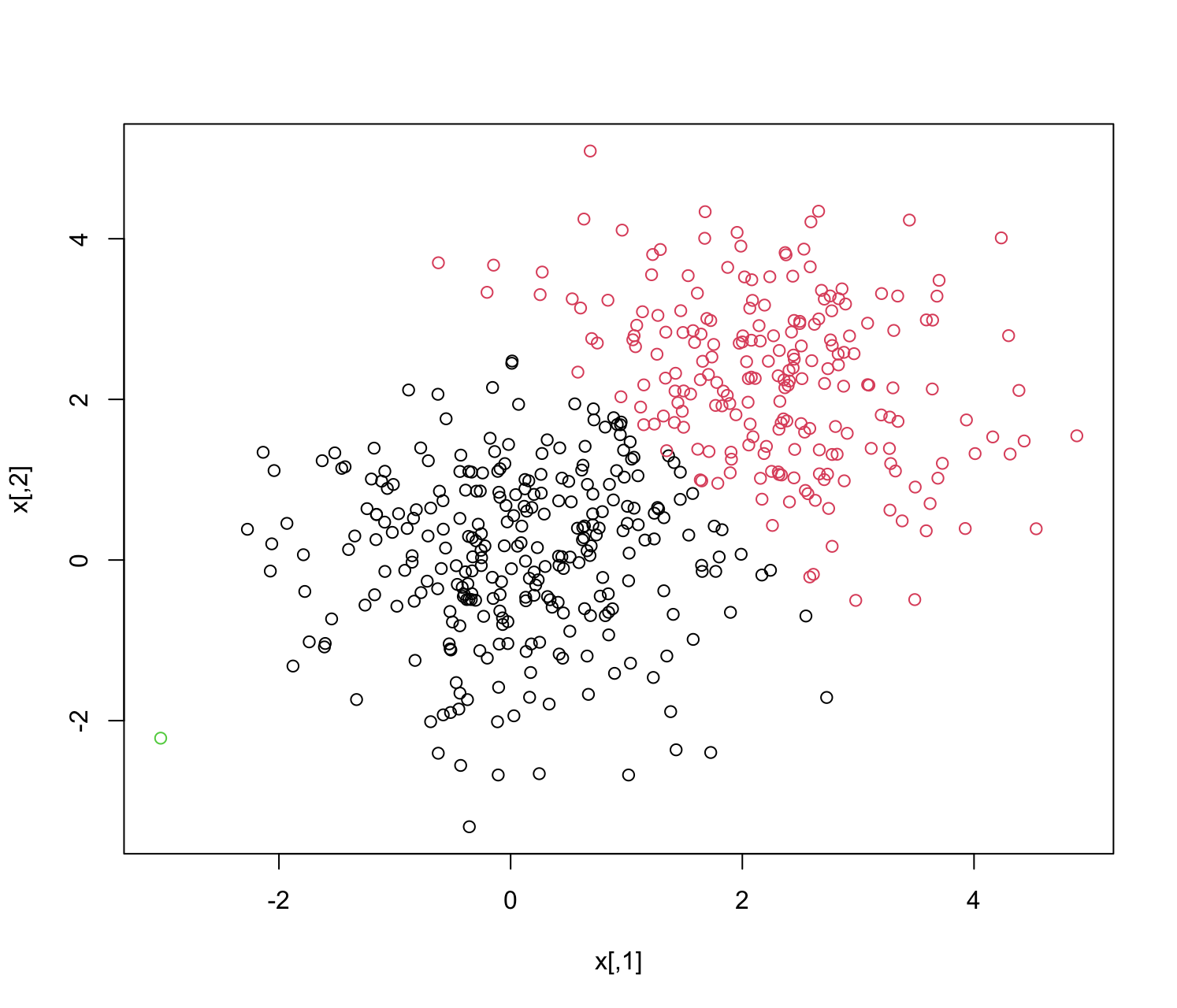} }}%
    \caption{Mean-shift with the Gaussian kernel converges, but identifies an incorrect number of clusters at the theoretically prescribed bandwidth.}%
    \label{fig:oneGauss}%
\end{figure}

 This dataset have two clusters, although the above result fails to identify the correct number of clusters. We illustrate this empirically in Figure \ref{fig:oneGauss} upon applying
mean-shift to a dataset with one half of the samples generated from an isotropic Gaussian and the
other half of samples generated from another isotropic Gaussian with a different mean.\\

\nd The failure of the Gaussian kernel is explained by the more general Lemma \ref{lem:largeBW-globalmean}, that heuristically shows whenever the derivative $k'(0)$ is finite, $k$ being the kernel profile, we cannot expect the right number of clusters for large bandwidth, as the mode estimate sequence is supposed to converge to the global sample mean: clearly this is the case for Gaussian kernels where $k(r):=e^{-r/2}.$

\begin{lemma}[\textbf{Large-bandwidth collapse of the first iterate to the global mean under finite $k'(0)$}]\label{lem:largeBW-globalmean}
Let $x_1,\dots,x_n\in\mathbb{R}^d$ and fix $i\in\{1,\dots,n\}$. For $h>0$, define
\[
g:=-k', \qquad
W_{ij}(h):=g\!\left(\frac{\|x_i-x_j\|^2}{h^2}\right),\qquad
x_i^{(1)}(h):=\frac{\sum_{j=1}^n W_{ij}(h)\,x_j}{\sum_{j=1}^n W_{ij}(h)}.
\]
Assume that $g:[0,\infty)\to(0,\infty)$ is continuous at $0$ and satisfies $0<g(0)<\infty$. Then
\[
\lim_{h\to\infty} x_i^{(1)}(h)=\frac1n\sum_{j=1}^n x_j.
\]
\end{lemma}

\begin{proof}
Set $r_{ij}(h):=\|x_i-x_j\|^2/h^2$. Then $r_{ij}(h)\to 0$ as $h\to\infty$, hence by continuity of $g$ at $0$,
$W_{ij}(h)=g(r_{ij}(h))\to g(0)$ for each $j$.
Let $S(h):=\sum_{j=1}^n W_{ij}(h)$. Then $S(h)\to ng(0)>0$, so $S(h)>0$ for all large $h$, and thus
$W_{ij}(h)/S(h)\to g(0)/(ng(0))=1/n$ for each $j$.
Since
$x_i^{(1)}(h)=\sum_{j=1}^n \bigl(W_{ij}(h)/S(h)\bigr)\,x_j$
is a finite sum, we conclude
$x_i^{(1)}(h)\to \sum_{j=1}^n (1/n)x_j=\frac1n\sum_{j=1}^n x_j$.
\end{proof}

\subsection{Correcting the failure of the Gaussian kernel with non-Gaussian, radially symmetric, strictly positive definite kernels}

\begin{lemma}[\textbf{Large-bandwidth first iterate under a power-law singularity of $g=-k'$ at the origin}]\label{lem:largeBW-powerlaw-local}
Let $k:[0,\infty)\to\mathbb{R}$ be $C^1$ on $(0,\infty)$, assume $k$ is nonincreasing on $(0,\infty)$, and set
$g(r):=-k'(r)\ge 0$ for $r>0$.
Assume there exist constants $C>0$ and $\beta>0$ such that
\[
g(r)=C\,r^{-\beta}\bigl(1+o(1)\bigr)\qquad (r\downarrow 0).
\]
Let $x_1,\dots,x_n\in\mathbb{R}^d$ and fix $i\in\{1,\dots,n\}$.
Assume $x_j\neq x_i$ for all $j\neq i$.
For $h>0$, define for $j\neq i$
\[
W_{ij}(h):=g\!\left(\frac{\|x_i-x_j\|^2}{h^2}\right),\qquad
\text{the first iterate }\widetilde x_i^{(1)}(h):=\frac{\sum_{j\neq i} W_{ij}(h)\,x_j}{\sum_{j\neq i} W_{ij}(h)}.
\]
Set $p:=2\beta>0$. Then
\[
\lim_{h\to\infty}\widetilde x_i^{(1)}(h)
=
\frac{\sum_{j\neq i}\|x_i-x_j\|^{-p}\,x_j}{\sum_{j\neq i}\|x_i-x_j\|^{-p}}.
\]
Moreover, the limiting weights $\|x_i-x_j\|^{-p}$ are strictly decreasing in $\|x_i-x_j\|$.
\end{lemma}

\begin{proof}
Fix $j\neq i$ and set $d_{ij}:=\|x_i-x_j\|>0$ and $r_{ij}(h):=d_{ij}^2/h^2$.
Then $r_{ij}(h)\downarrow 0$ as $h\to\infty$ and hence
\[
W_{ij}(h)=g(r_{ij}(h))=C\,r_{ij}(h)^{-\beta}\bigl(1+o(1)\bigr)
= C\,h^{2\beta}d_{ij}^{-2\beta}\bigl(1+o(1)\bigr).
\]
Since the index set $\{j\neq i\}$ is finite and $r_{ij}(h)\to 0$ for each such $j$, the $o(1)$ term is
uniform in $j\neq i$. Thus
\[
\sum_{j\neq i} W_{ij}(h)\,x_j
= C\,h^{2\beta}\sum_{j\neq i} d_{ij}^{-2\beta}\bigl(1+o(1)\bigr)x_j,\qquad
\sum_{j\neq i} W_{ij}(h)
= C\,h^{2\beta}\sum_{j\neq i} d_{ij}^{-2\beta}\bigl(1+o(1)\bigr).
\]
Because $g\ge 0$ and $j\neq i$ ensures $W_{ij}(h)>0$, the denominator is strictly positive for all $h$.
Cancel $C\,h^{2\beta}$ and let $h\to\infty$ to obtain
\[
\widetilde x_i^{(1)}(h)\to
\frac{\sum_{j\neq i} d_{ij}^{-2\beta}\,x_j}{\sum_{j\neq i} d_{ij}^{-2\beta}}
=
\frac{\sum_{j\neq i}\|x_i-x_j\|^{-p}\,x_j}{\sum_{j\neq i}\|x_i-x_j\|^{-p}},
\qquad p=2\beta.
\]
Finally, $d\mapsto d^{-p}$ is strictly decreasing on $(0,\infty)$ for $p>0$. \qedhere
\end{proof}

\nd The above Lemma \ref{lem:largeBW-powerlaw-local} implies that for large bandwidths $h$, the weights we put on $x_j$ to calculate the first iterate $\widetilde x_i^{(1)}(h)$ is inversely proportional to $\|x_i-x_j\|^{p},$ which seems to imply that unlike Gaussian kernel, there is more chance to obtain more \textit{localized basins} of attraction if we use the above types of non-Gaussian kernels, something we will also obsrve in the experiments. However, the following lemma tells us that within the class of radially symmetric, strictly positive definite kernels whose profiles are given by completely monotone functions (cf. Definition \ref{Completely monotone functions}), there is a limit to $p$ above, namely $p\le 2.$ This means we can make the basins of attraction localized up to a certain extent only.

\begin{lemma}[\textbf{Radially symmetric, strictly positive definite kernels cannot yield distance exponents $p>2$ in Lemma \ref{lem:largeBW-powerlaw-local}}]\label{lem:CM-bound-p}
Assume in addition that $k$ is \emph{completely monotone} on $(0,\infty)$ (cf. Definition \ref{Completely monotone functions}).
Then
\[
g(r)\le \frac{k(0)}{e}\,\frac{1}{r}\qquad\text{for all }r>0.
\]
In particular, if $g$ satisfies the power-law singularity assumption
$g(r)=C\,r^{-\beta}\bigl(1+o(1)\bigr)$ as $r\downarrow 0$ for some $C>0$ and $\beta>0$,
then necessarily $\beta\le 1$, and hence the distance exponent $p=2\beta$ in
Lemma \ref{lem:largeBW-powerlaw-local} satisfies $p\le 2$.
\end{lemma}

\begin{proof}
By Proposition \ref{prop:HBD thm}, complete monotonicity of $k$ implies that there exists a finite
positive Borel measure $\mu$ on $[0,\infty)$ such that
\[
k(r)=\int_0^\infty e^{-rt}\,d\mu(t)\qquad (r>0).
\]
Differentiating under the integral sign yields
\[
g(r)=-k'(r)=\int_0^\infty t\,e^{-rt}\,d\mu(t)\qquad (r>0).
\]
Since $\sup_{t\ge 0} t e^{-rt} = \frac{1}{er}$ for each $r>0$, we obtain
\[
g(r)\le \frac{1}{er}\int_0^\infty d\mu(t)=\frac{1}{er}\,k(0)\qquad (r>0),
\]
which proves the first claim.

Now assume $g(r)=C\,r^{-\beta}(1+o(1))$ as $r\downarrow 0$ with $\beta>0$.
If $\beta>1$, then $r\,g(r)\to\infty$ as $r\downarrow 0$, contradicting the bound
$r\,g(r)\le k(0)/e$ for all $r>0$. Hence $\beta\le 1$, and therefore $p=2\beta\le 2$.
\end{proof}

\begin{remark}[\textbf{Laplace, stretched exponential and Cauchy-type kernels as instances of Lemma \ref{lem:largeBW-powerlaw-local}}]
\label{rmk:powerlaw-instances}
 Recall that in Remark \ref{rmk:examples-CMF-experiments}, we showed that the Laplace and stretched exponential kernels as examples of radially symmetric, positive definite kernels. Below we show that they satisfy the hypotheses of Lemma \ref{lem:largeBW-powerlaw-local} above.

\emph{(i) Laplace kernel.} For the Laplace kernel, the profile is $k(q)=\exp(-\lambda\sqrt q)$ with $\lambda>0$, then
\[
g(r)=-k'(r)=\frac{\lambda}{2\sqrt r}\exp(-\lambda\sqrt r)
=\frac{\lambda}{2}\,r^{-1/2}\bigl(1+o(1)\bigr)\qquad (r\downarrow 0).
\]
Thus Lemma~\ref{lem:largeBW-powerlaw-local} applies with $\beta=\tfrac12$ (so $p=1$), yielding an inverse-distance
weighted local mean.

\emph{(ii) Stretched exponential kernel.} If $k(q)=\exp(-\lambda q^\alpha)$ with $\lambda>0$ and $\alpha\in(1/2,1)$, then
\[
g(r)=-k'(r)=\lambda\alpha\,r^{\alpha-1}\exp(-\lambda r^\alpha)
=\lambda\alpha\,r^{-(1-\alpha)}\bigl(1+o(1)\bigr)\qquad (r\downarrow 0).
\]
Thus Lemma~\ref{lem:largeBW-powerlaw-local} applies with $\beta=1-\alpha$ (so $p=2(1-\alpha)$), yielding the
distance-weighted local mean.

\emph{(iii) Cauchy-type kernel.}
For the completely monotone Cauchy-type profile
\[
k(q)=\frac{1}{1+q^{\alpha}}, 
\qquad 0<\alpha\le 1,
\]
with derivative
\[
g(q)=-k'(q)
=\frac{\alpha\, q^{\alpha-1}}{(1+q^{\alpha})^{2}}
\sim \alpha\, q^{\alpha-1}
=\alpha\, q^{-P/2}
\qquad (q\downarrow 0),
\]
where 
\[
P=2(1-\alpha),
\]
Lemma~\ref{lem:largeBW-powerlaw-local} applies with 
\[
\beta=1-\alpha=\frac{P}{2}.
\]

\end{remark}

\subsection{Experiments with synthetic data}

\subsection*{Large-bandwidth regime: $h = 10\|x_{\max}\|$}

We consider two well-separated Gaussian clusters in $\mathbb{R}^2$
($n=300$, separation $5.0$, noise $\sigma=0.35$) and fix
\[
h = 10\|x_{\max}\|.
\]
This choice \textit{intentionally} places the experiment in the \textit{extreme large-bandwidth regime} predicted by the theory developed in earlier sections. All kernels use exclude-self profile mean shift with fixed numerical
regularization $\varepsilon_q=10^{-12}$ for singular profiles.
Mode-merging tolerances are fixed once per kernel family and are not tuned per run.

\paragraph{\textbf{Gaussian kernel:} asymptotic weight homogenization and collapse.}

For the Gaussian profile $k(q)=e^{-q/2}$,
\[
K=1, \qquad \text{accuracy}=50\%.
\]
All trajectories converge in approximately $6$ iterations to a single
global barycenter located between the two true clusters.

This behavior is consistent with the large-bandwidth asymptotics:
when $h \gg$ inter-cluster separation, 
\[
q_{ij}=\frac{\|x_i-x_j\|^2}{h^2} \ll 1
\quad\Longrightarrow\quad
g(q_{ij}) \approx \tfrac12,
\]
so all weights become nearly constant.
The update therefore reduces to a \textit{global averaging step}, forcing collapse to a \textit{single mode}.
The collapse is therefore a consequence of weight homogenization,
not of mean shift per se.

\paragraph{\textbf{Laplace kernel: persistence of block structure.}}

For the completely monotone Laplace profile
\[
k(q)=e^{-\sqrt{q}},
\]
we obtain
\[
K=2, \qquad \text{accuracy}=100\%.
\]
Although convergence is slower (average $\approx 193$ iterations),
the trajectories remain confined to their respective clusters.
No collapse occurs.

\paragraph{\textbf{Cauchy-type kernel: power-law singular reinforcement.}}

For the completely monotone Cauchy-type profile
\[
k(q)=\frac{1}{1+q^\alpha},
\qquad \alpha = 1-\frac{P}{2}, \quad P=1.2,
\]
with derivative
\[
g(q)=\alpha\,\frac{q^{\alpha-1}}{(1+q^\alpha)^2}
\sim \alpha q^{-P/2} \quad (q\downarrow 0),
\]
we again obtain
\[
K=2, \qquad \text{accuracy}=100\%.
\]
Convergence is substantially faster than Laplace
(average $\approx 52$ iterations).
The block structure in the weight matrices is visibly sharper
than for the Gaussian case and remains stable over iterations.
The singular power-law behavior reinforces near-neighbor influence,
preventing homogenization of weights even when $h$ is very large. The collapse of the Gaussian case is consistent with near-homogenization of the row-normalized weights at this scale.

\paragraph{Interpretation.}

At $h = 10\|x_{\max}\|$ we observe:

\begin{itemize}
\item \textbf{Gaussian profile:} near-constant weights $\Rightarrow$ global averaging $\Rightarrow$ single basin.
\item \textbf{Laplace (CM):} non-constant relative weighting $\Rightarrow$ two stable basins.
\item \textbf{Cauchy-type:} singular reinforcement of local influence $\Rightarrow$ two stable basins with faster convergence.
\end{itemize}

Thus, large-bandwidth collapse is not an inherent property of mean shift,
but a consequence of exponential weight flattening.
Completely monotone power-law profiles maintain sufficient relative
weight asymmetry to preserve multiple basins even in an extreme
global-bandwidth regime.

\paragraph{\textbf{Distribution-level comparison via weak distances.}}

In addition to reporting the number of merged modes and clustering
accuracy, we complement the above analysis with a distribution-level
comparison based on weak distances between empirical distribution
functions. This follows the idea of distribution-function-based
predictive accuracy discussed in \cite{GiudiciRaffinetti2025}, see also the relevant paper \cite{GuidottiRGA2025}.

In the clustering setting considered here the class labels are nominal, so rank-based comparisons of the labels themselves are not meaningful. Since cluster labels are arbitrary up to permutation, direct comparison of label values is not well defined; comparing the empirical distributions of cluster sizes instead yields a permutation-invariant summary of the clustering structure.
Instead we compare the empirical cumulative distribution functions
(CDFs) of the \emph{normalized cluster-size distributions}.
Let
\[
p=\left\{\frac{|C_1|}{n},\ldots,\frac{|C_K|}{n}\right\}
\]
denote the normalized sizes of the predicted clusters and
\[
q=\left\{\frac{|T_1|}{n},\ldots,\frac{|T_m|}{n}\right\}
\]
the normalized sizes of the ground-truth classes. Here $C_1,\ldots,C_K$ denote the clusters produced by mean shift, $|C_i|$ is the number of data points assigned to cluster $C_i$, and $T_1,\ldots,T_m$ denote the ground-truth classes with sizes $|T_j|$.
We then compute the Cram\'er--von Mises (CvM) weak distance
\[
\mathrm{CvM}(F_p,F_q)
=
\int_0^1 (F_p(u)-F_q(u))^2\,du .
\]

\nd Since $p$ and $q$ are normalized cluster-size distributions, their cumulative distribution functions are supported on $[0,1]$, which motivates the integration range.

\nd For the present synthetic experiment the ground-truth classes are
balanced, $q=(0.5,0.5)$. The resulting CvM discrepancies (\cite{GiudiciRaffinetti2025}) are

\[
\begin{array}{lccc}
\text{Kernel} & K & \text{Accuracy} & \mathrm{CvM}(F_p,F_q)\\
\hline
\text{Gaussian} & 1 & 50\% & 0.5 \\
\text{Laplace} & 2 & 100\% & 0.0 \\
\text{Cauchy-type} & 2 & 100\% & 0.0 .
\end{array}
\]

Thus the weak-distance comparison is consistent with the qualitative
behavior observed above: the Gaussian kernel collapses to a single
cluster, producing a substantial discrepancy between predicted and
true cluster-size distributions, whereas both the Laplace and
Cauchy-type kernels recover two balanced clusters and therefore yield
vanishing CvM discrepancy.
\noindent
Figure \ref{fig:large_bandwidth_comparison_synthetic} below shows the trajectory plots for the large-bandwidth experiment $h = 10\|x_{\max}\|$: first, the Gaussian kernel illustrating collapse to a single basin, followed by the corresponding Laplace and Cauchy-type kernel results demonstrating non-collapse and correct recovery of two clusters.

\begin{figure}[ht]
\centering

\begin{subfigure}[t]{0.45\linewidth}
    \centering
    \includegraphics[width=\linewidth]{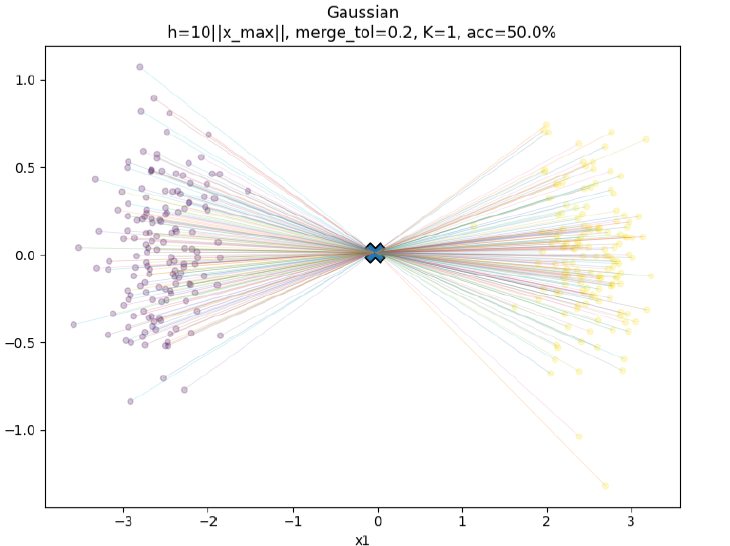}
    \caption{Gaussian: synthetic data (collapse)}
\end{subfigure}
\hfill
\begin{subfigure}[t]{0.45\linewidth}
    \centering
    \includegraphics[width=\linewidth]{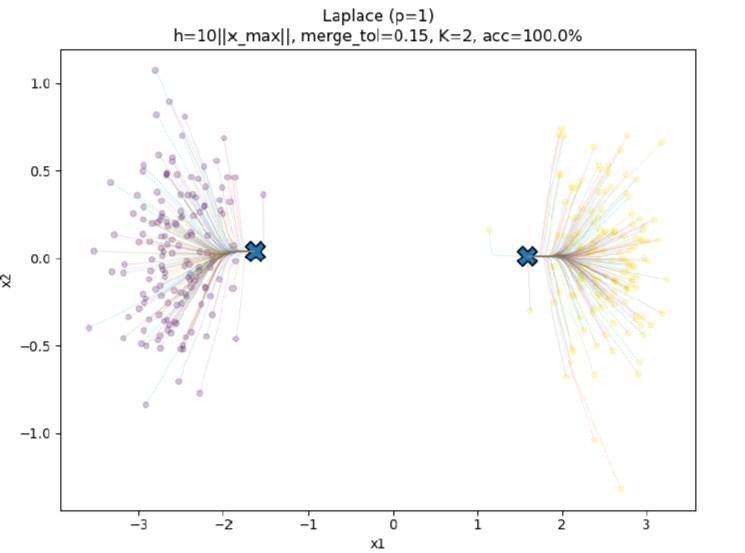}
    \caption{Laplace: synthetic data (non-collapse)}
\end{subfigure}
\hfill
\begin{subfigure}[t]{0.45\linewidth}
    \centering
    \includegraphics[width=\linewidth]{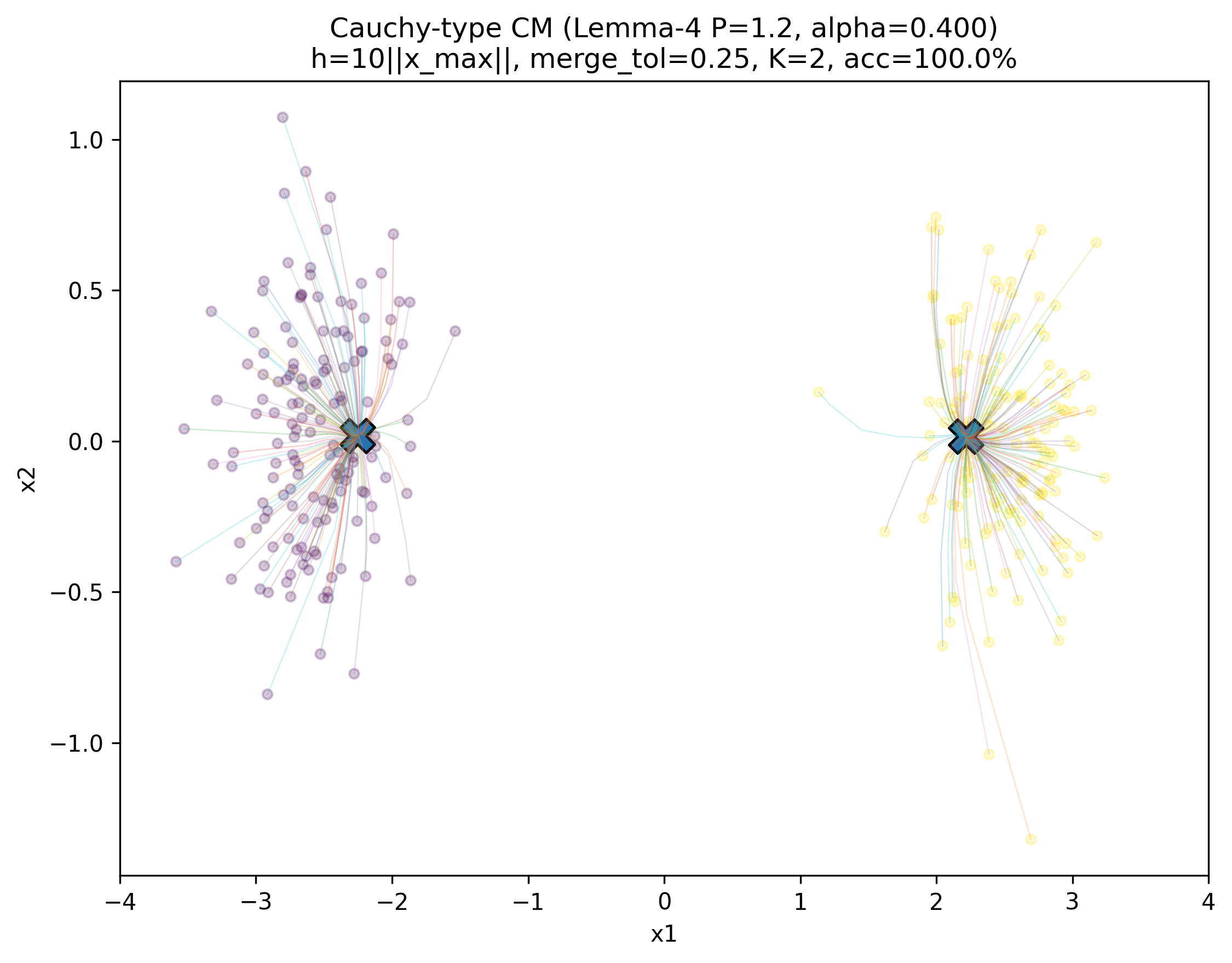}
    \caption{Cauchy-type: synthetic data (non-collapse)}
\end{subfigure}

\caption{Large-bandwidth experiment $h = 10\|x_{\max}\|$: Gaussian collapse versus Laplace and Cauchy-type kernels for synthetic data.}
\label{fig:large_bandwidth_comparison_synthetic}

\end{figure}

\subsection{Experiments with real data}
\subsection{Iris dataset}
\subsection*{Iris dataset: extreme large-bandwidth regime $h = 10\|x_{\max}\|$}

We consider the full Iris dataset ($n=150$, three ground-truth classes),
after feature-wise $z$-score normalization.
No PCA is used for the dynamics.
We fix
\[
R = \max_i \|x_i\|, 
\qquad 
h = 10R,
\]
which yields $R = 3.537642$ and $h = 35.376423$.
All experiments use the standard exclude-self mean shift update, 
since for Laplace and power-law Cauchy-type profiles 
the weight function $g(q)=-k'(q)$ behaves like $q^{\alpha-1}$ 
near $q=0$ and is singular when $\alpha<1$. 
Thus, the self-weight would be \textit{infinite}.
Therefore, the current iterate is removed from the weighted average at each step.

We use a fixed mode-merging tolerance $\texttt{merge\_tol}=0.05$.

Weights are always defined as $g(q)=-k'(q)$, as in the MS algorithm described in Section \ref{scn:Mathematical description of the MS Algorithm}. In what follows, we report the results from Gaussian, Laplace and Cauchy-type kernels. Figure \ref{fig:iris} demonstrates the trajectories.

\paragraph{\textbf{Gaussian kernel (exponential profile).}}

For the Gaussian profile $k(q)=e^{-q/2}$ we obtain
\[
K_{\text{merged}} = 1, 
\qquad 
\mathrm{ARI}=0.0000,
\qquad 
\mathrm{accuracy}=0.3333.
\]
All $150$ trajectories converge to a single global barycenter.
The trajectory plots (20 seeds per true class) show complete collapse:
all paths contract toward a single interior point and the merged mode size is $150$.
The weight matrix $W_{ij}=g(\|x_i-x_j\|^2/h^2)$ is nearly constant off-diagonal:
both the raw and label-sorted heatmaps exhibit no visible block structure,
consistent with large-bandwidth weight homogenization and a single basin.

\paragraph{\textbf{Laplace kernel (completely monotone, exponential in $\sqrt{q}$).}}

For the Laplace profile $k(q)=e^{-\sqrt{q}}$ we obtain
\[
K_{\text{merged}} = 1, 
\qquad 
\mathrm{ARI}=0.0000,
\qquad 
\mathrm{accuracy}=0.3333.
\]
Trajectories again collapse to a single interior point, yielding one basin.
The raw and label-sorted weight-matrix heatmaps do not display a persistent
three-block structure aligned with the ground-truth classes at $h=10R$.
Accordingly, no class-level recovery occurs in this extreme global-bandwidth regime.

\paragraph{\textbf{Cauchy-type completely monotone kernel (power-law profile).}}

For the Cauchy-type profile
\[
k(q)=\frac{1}{1+q^\alpha}, 
\qquad 
\alpha = 1-\frac{P}{2}, 
\quad P=1.99,
\]
with derivative
\[
g(q)=\alpha\,\frac{q^{\alpha-1}}{(1+q^\alpha)^2},
\]
we obtain
\[
K_{\text{merged}} = 6,
\qquad 
\mathrm{ARI}=0.5148,
\qquad 
\mathrm{accuracy}=0.6733.
\]
The trajectory plots (20 seeds per true class) exhibit multiple distinct attraction
basins, i.e., no global collapse.
The six merged basins are not perfectly aligned with the three ground-truth classes,
but they yield substantially higher agreement than the collapsed cases.

\paragraph{Summary of large-bandwidth behavior on Iris.}

At $h = 10\|x_{\max}\|$:
\begin{itemize}
\item \textbf{Gaussian:} $K=1$, ARI $=0.0000$, accuracy $=0.3333$; trajectories collapse; weights homogenize.
\item \textbf{Laplace:} $K=1$, ARI $=0.0000$, accuracy $=0.3333$; trajectories collapse; no persistent label-block structure.
\item \textbf{Cauchy-type CM:} $K=6$, ARI $=0.5148$, accuracy $=0.6733$; multiple basins persist; weights remain non-uniform.
\end{itemize}

\paragraph{\textbf{Distribution-level comparison via weak distances.}}

To complement the ARI and clustering accuracy reported above, we also
perform the same distribution-level comparison used in the synthetic-data
experiment, namely the Cram\'er--von Mises (CvM) weak distance between the
empirical distribution functions of the normalized predicted cluster-size
distribution and the normalized ground-truth class-size distribution.
The definition of the weak distance and the notation $F_p,F_q$ were
introduced in the synthetic-data section and are therefore not repeated here.

For the Iris dataset the three ground-truth classes are balanced,
so the normalized class-size distribution is
\[
q=(1/3,\,1/3,\,1/3).
\]

For the Gaussian and Laplace kernels the algorithm collapses to a single
merged basin containing all $150$ observations, yielding the predicted
cluster-size distribution
\[
p=(1).
\]

For the Cauchy-type kernel the algorithm produces $K=6$ merged basins.
The corresponding predicted cluster-size distribution is determined by
the merged-mode sizes returned by the endpoint-merging procedure.

The resulting CvM discrepancies are summarized below.

\[
\begin{array}{lccc}
\text{Kernel} & K & \text{Accuracy} & \mathrm{CvM}(F_p,F_q)\\
\hline
\text{Gaussian} & 1 & 0.3333 & 0.3333 \\
\text{Laplace} & 1 & 0.3333 & 0.3333 \\
\text{Cauchy-type} & 6 & 0.6733 & 0.2126
\end{array}
\]

Thus the weak-distance comparison is consistent with the qualitative
behavior observed in the trajectory plots: the Gaussian and Laplace
kernels collapse to a single basin, producing a substantial discrepancy
between predicted and true class-size distributions, whereas the
Cauchy-type kernel maintains multiple attraction basins and therefore
yields a smaller weak-distance discrepancy.

Figure \ref{fig:iris} depicts the trajectories for these three kernels.

\begin{figure}[ht]
\centering

\begin{subfigure}[t]{0.45\linewidth}
    \centering
    \includegraphics[width=\linewidth]{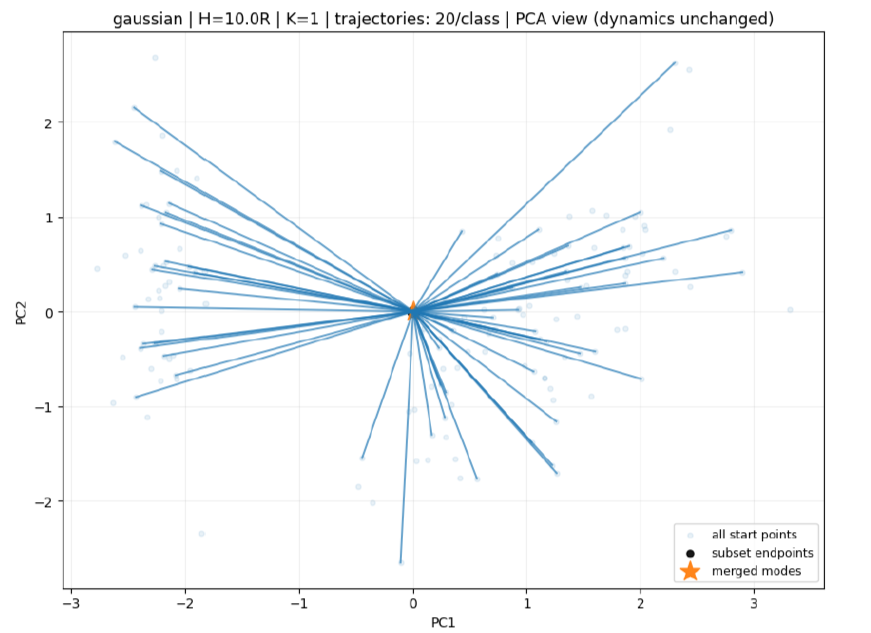}
    \caption{Gaussian MS on Iris (collapse)}
\end{subfigure}
\hfill
\begin{subfigure}[t]{0.45\linewidth}
    \centering
    \includegraphics[width=\linewidth]{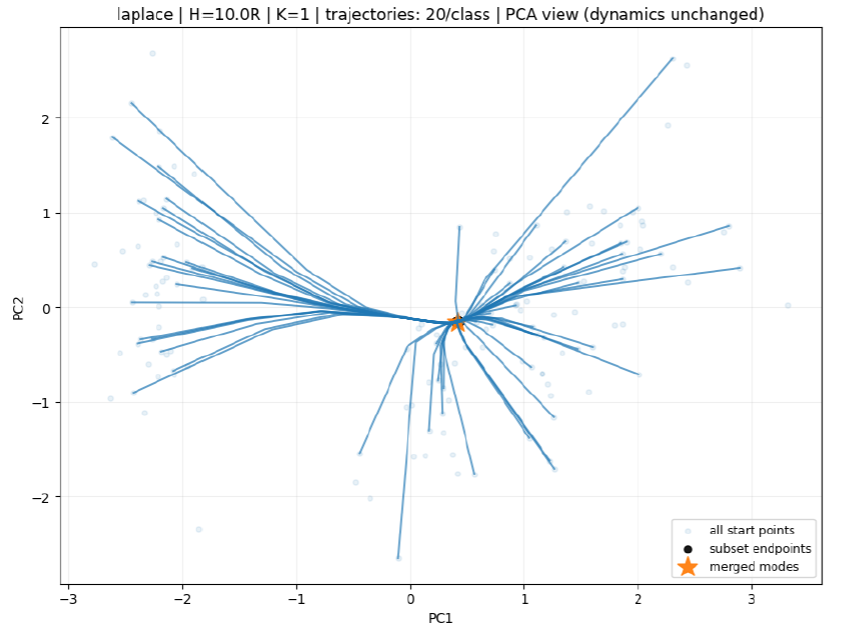}
    \caption{Laplace MS on Iris (collapse)}
\end{subfigure}
\hfill
\begin{subfigure}[t]{0.45\linewidth}
    \centering
    \includegraphics[width=\linewidth]{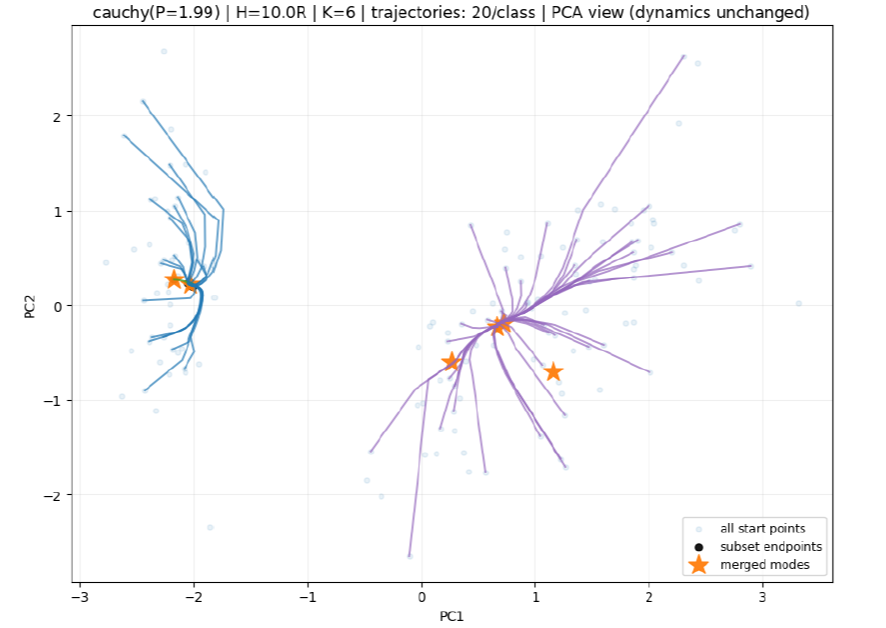}
    \caption{Cauchy-type on MS on Iris (non-collapse)}
\end{subfigure}

\caption{Large-bandwidth experiment $h = 10\|x_{\max}\|$ on Iris: Gaussian and Laplace kernel collapse Cauchy-type non-collapse.}
\label{fig:large_bandwidth_comparison}
\label{fig:iris}
\end{figure}
\subsection{Wheat Seeds Dataset}

We consider the UCI Wheat Seeds dataset ($n=210$, $D=7$) with three balanced classes ($70$ samples each). All features are standardized by $z$-score normalization and the data are projected onto $\mathrm{PCA}(2)$. Mean shift is then performed in this two-dimensional space using derivative weights
\[
W_{ij}=g(q_{ij}), \qquad
q_{ij}=\frac{\|x_i-x_j\|^2}{H^2}, \qquad
g(q)=-k'(q), \qquad x_i \in \cl^2 \text{ are PCA coordinates}.
\]
The bandwidth is chosen as
\[
H = 10\max_i\|x_i\|,
\]
with the norm computed in the $\mathrm{PCA}(2)$ space, placing the experiment in the extreme large-bandwidth regime.

\paragraph{Gaussian kernel.}
For the Gaussian profile
\[
k(q)=e^{-q/2}, \qquad g(q)=-k'(q)=\tfrac12 e^{-q/2},
\]
all trajectories collapse to a single attractor. After merging endpoints we obtain
\[
K_{\mathrm{merged}}=1,
\qquad
\mathrm{ARI}=0.0000,
\qquad
\mathrm{accuracy}=0.3333,
\]
with cluster size distribution $[210]$.

\paragraph{Cauchy-type completely monotone kernel.}
We next consider the Cauchy-type family
\[
k(q)=\frac{1}{1+q^\alpha}, \qquad
\alpha = 1-\frac{P}{2}, \qquad P=1.99,
\]
with derivative weights
\[
g(q)=-k'(q)=
\alpha\frac{q^{\alpha-1}}{(1+q^\alpha)^2}.
\]
In this case the trajectories do not collapse globally. After endpoint merging we obtain
\[
K_{\mathrm{merged}}=22,
\qquad
\mathrm{ARI}=0.3721,
\qquad
\mathrm{accuracy}=0.9048.
\]
The merged cluster sizes are
\[
[60,32,18,14,13,11,9,8,7,6,6,5,5,3,3,2,2,2,1,1,1,1].
\]

\paragraph{Interpretation.}
In the large-bandwidth regime the Gaussian kernel collapses to a single basin of attraction, producing no meaningful clustering structure. The Cauchy-type kernel avoids this collapse and preserves multiple attraction basins. Although this leads to over-segmentation relative to the three true classes, many of the resulting basins remain class-pure, which explains the high many-to-one accuracy.

\paragraph{\textbf{Distribution-level comparison via CvM distance.}}
To complement ARI and clustering accuracy, we also compare the empirical distribution of predicted cluster sizes with the ground-truth class-size distribution. Since the three true classes are balanced, the normalized ground-truth distribution is
\[
q=(1/3,\,1/3,\,1/3).
\]
Let $p$ denote the normalized predicted cluster-size distribution. We measure the discrepancy between the empirical cumulative distribution functions $F_p$ and $F_q$ using the Cramér–von Mises weak distance (cf. \cite{GiudiciRaffinetti2025})
\[
\mathrm{CvM}(F_p,F_q)=\int_{0}^{1} (F_p(u)-F_q(u))^2\,du .
\]

\[
\begin{array}{lcccc}
\text{Kernel} & K_{\mathrm{merged}} & \mathrm{ARI} & \mathrm{accuracy} & \mathrm{CvM}(F_p,F_q) \\
\hline
\text{Gaussian} & 1 & 0.0000 & 0.3333 & 0.3333 \\
\text{Cauchy-type} & 22 & 0.3721 & 0.9048 & 0.2376
\end{array}
\]

The CvM distance therefore provides a complementary distribution-level comparison. The Gaussian collapse produces the largest discrepancy between predicted and true cluster-size distributions, whereas the Cauchy-type kernel yields a smaller discrepancy by maintaining multiple attraction basins.

\begin{figure}[ht]
\centering

\begin{subfigure}{0.45\linewidth}
\centering
\includegraphics[width=\linewidth]{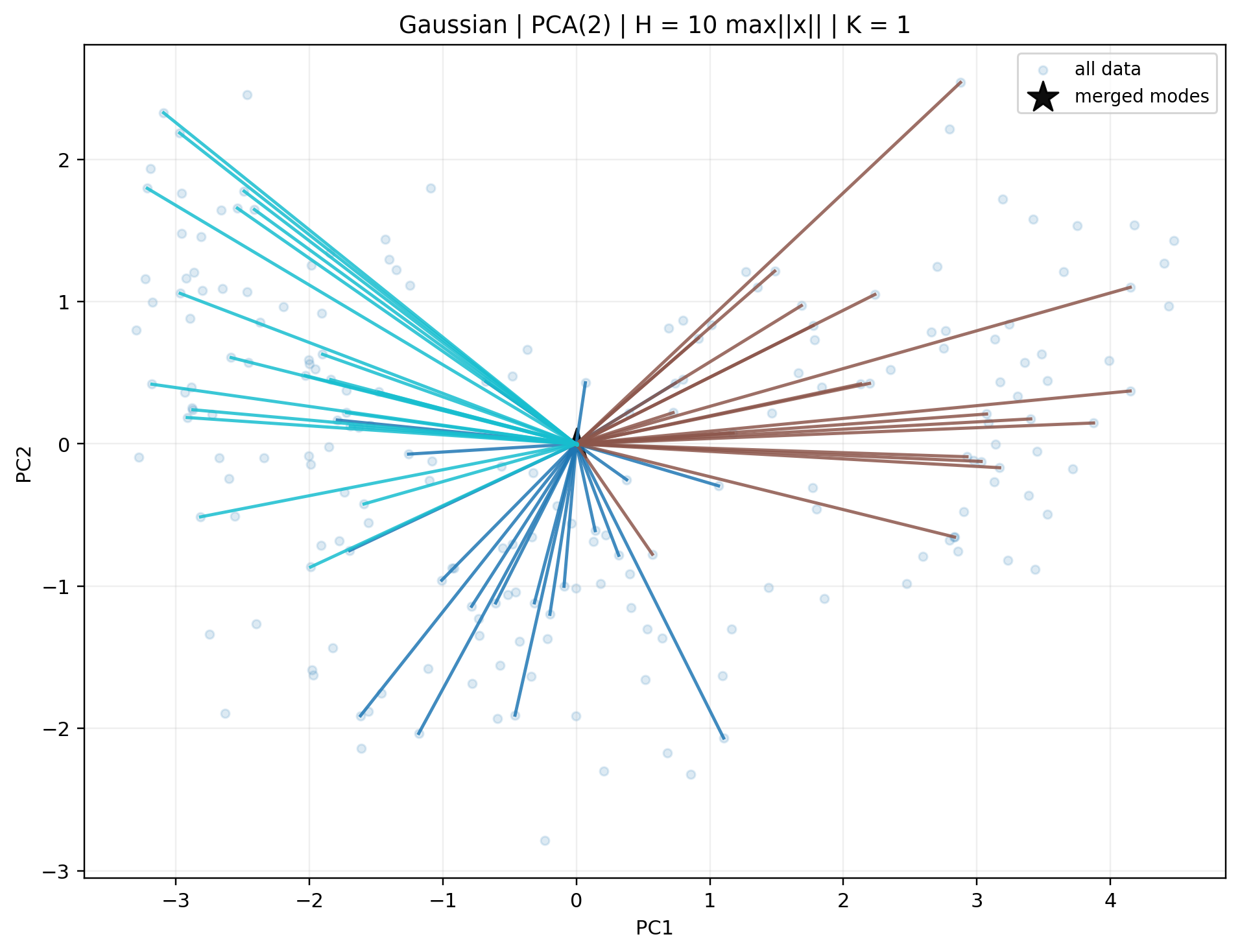}
\caption{Gaussian on Wheat Seed ($K=1$)}
\end{subfigure}
\hfill
\begin{subfigure}{0.45\linewidth}
\centering
\includegraphics[width=\linewidth]{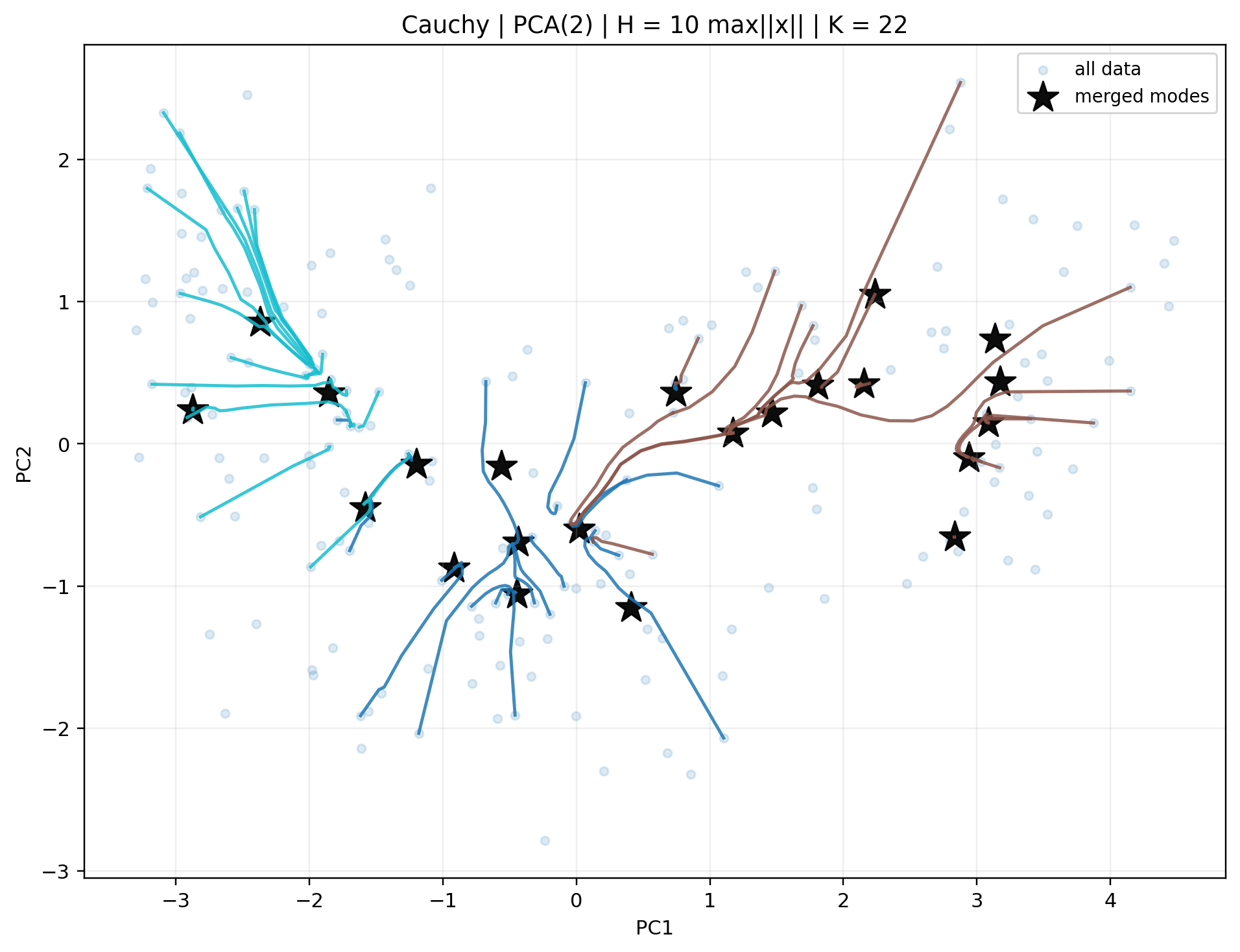}
\caption{Cauchy-type CM on Wheat Seed ($K=22$)}
\end{subfigure}

\caption{Mean-shift trajectories in the PCA$(2)$ projection of Wheat Seed data at $H = 10\max_i \|x_i\|$.}
\label{fig:pca2_panel}

\end{figure}

\subsection{Image segmentation experiment: grayscale elephant image}

To assess the behavior of the completely monotone\textit{ Cauchy-type kernel} on real data, we conducted a mean shift segmentation experiment on a grayscale image of an elephant (see Figure~\ref{fig:elephant_segmentation}). The image was treated as a point cloud in feature space, and clustering was performed using the Cauchy-type profile
\[
k(q) = \frac{1}{1 + q^\alpha}, 
\qquad \alpha = 1 - \frac{P}{2},
\]
with parameter $P=1.99$ (close to the upper limit $P<2$ allowed by the completely monotone framework, cf. Lemma \ref{lem:CM-bound-p}). 

We deliberately did not use the Laplace kernel in this experiment. In earlier Iris experiment, the Laplace profile showed weaker separation properties at large bandwidths. The Gaussian profile was also excluded due to its demonstrated large-bandwidth collapse behavior even for synthetic datasets. The Cauchy-type profile was selected because its power-law singularity near $q=0$ preserves stronger local weighting and mitigates global weight homogenization.

The experiment was run with bandwidth multiplier $H_{\text{mult}}=10$ (so just like the previous two experiments, bandwidth $h:=10\norm{x}_{max}$) and $k$-nearest-neighbor restriction $k=300$. After convergence, mode merging was performed using a tolerance determined from the median nearest-neighbor distance (median $\approx 0.1842$, merge tolerance $\approx 0.7368$), resulting in $K_{\text{merged}}=174$ clusters. A subsequent intensity-based dark-mode selection (quantile threshold $0.35$) was applied to produce a binary segmentation mask.

The resulting segmentation mask visually separates the elephant from substantial portions of the background without artificial collapse into a single basin. Since no ground-truth segmentation labels are available for this example, we report the result qualitatively. This experiment demonstrates that, in contrast to the Gaussian case, the Cauchy-type completely monotone kernel maintains sufficient locality to produce structured multi-basin behavior on real grayscale image data.

Here are the original elephant image and its segmented image.

\begin{figure}[ht]
\centering

\begin{subfigure}[t]{0.48\linewidth}
    \centering
    \includegraphics[width=\linewidth]{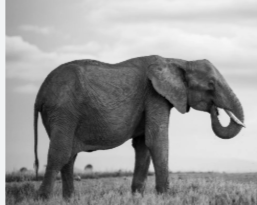}
    \caption{Elephant: original image}
\end{subfigure}
\hfill
\begin{subfigure}[t]{0.48\linewidth}
    \centering
    \includegraphics[width=\linewidth]{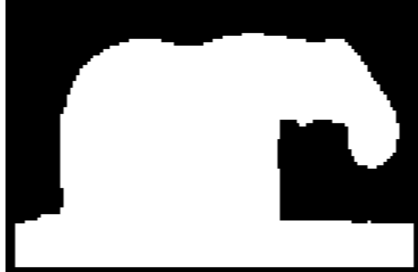}
    \caption{Elephant: binary segmentation mask}
\end{subfigure}

\caption{Grayscale elephant image and its segmentation obtained using the Cauchy-type kernel.}
\label{fig:elephant_segmentation}
\end{figure}

\nd N.B. Since no ground-truth segmentation is available in this experiment, we do not report a distribution-level CvM comparison (cf. \cite{GiudiciRaffinetti2025}) here.

\section*{Funding}
This research was supported by Research Foundation – Flanders (FWO) under the Odysseus II programme (Grant No. G0DBZ23N).
\section*{Conflict of interest statement}
The corresponding author states that there is no conflict of interest.
\section*{Data Availability Statement}
The corresponding author states that all the data used in our experiments are publicly available datasets.

\end{document}